\documentclass[journal]{IEEEtran}
\usepackage{amsmath,graphicx,amsfonts,amssymb,amsthm,cite,algorithmic,array,stfloats,url,subfig,pgfplots,tikz,stmaryrd,booktabs,multirow,makecell,diagbox,slashbox}
\definecolor{hookersgreen}{rgb}{0.0, 0.44, 0.14} 
\definecolor{crimson}{rgb}{0.86, 0.08, 0.24}
\definecolor{ceruleanblue}{rgb}{0.16, 0.32, 0.75}
\definecolor{azure}{rgb}{0.0, 0.5, 1.0}
\definecolor{frenchblue}{rgb}{0.0, 0.45, 0.73}
\definecolor{forestgreen(traditional)}{rgb}{0.0, 0.27, 0.13}
\definecolor{mygreen}{rgb}{0.09, 0.45, 0.27}
\definecolor{myblue}{rgb}{0.2383,0.5195,0.7734}
\definecolor{mygreen}{rgb}{0.6445,0.9297,0.0039}
\definecolor{darklavender}{rgb}{0.45, 0.31, 0.59}
\definecolor{americanrose}{rgb}{1.0, 0.01, 0.24}
\definecolor{pigblue}{rgb}{0.2, 0.2, 0.6}
\definecolor{blue(ryb)}{rgb}{0.01, 0.28, 1.0}
\definecolor{amethyst}{rgb}{0.6, 0.4, 0.8}
\definecolor{deepmagenta}{rgb}{0.8, 0.0, 0.8}
\definecolor{carminered}{rgb}{1.0, 0.0, 0.22}
\definecolor{iris}{rgb}{0.35, 0.31, 0.81}

\newtheorem{lemma}{\hspace{0pt}\bf Lemma}
\newtheorem{proposition}{\hspace{0pt}\bf Proposition}

\newtheorem{theorem}{\hspace{0pt}\bf Theorem}
\newtheorem{corollary}{\hspace{0pt}\bf Corollary}

\newtheoremstyle{mystyle} 
    {\topsep}                    
    {\topsep}                    
    {\normalfont}                
    {}                           
    {\bfseries}                  
    {.}                          
    {.5em}                       
    {}                           
\theoremstyle{mystyle}
\newtheorem{remark}{Remark}

\newcommand{\ones}{\mathbf 1}
\newcommand{\reals}{{\mbox{\bf R}}}

\newcommand{\diag}{\mathop{\bf diag}}

\def \Cov     {\text{\normalfont Cov}  }

\def \diag    {\text{\normalfont diag} }

\newcommand{\Var}{\mathrm{Var}}

\def \reals    {{\mathbb R}}

\def\mbE{{\ensuremath{\mathbb E}}}

\def\mbR{{\ensuremath{\mathbb R}}}

\def\ccalA{{\ensuremath{\mathcal A}}}

\def\ccalD{{\ensuremath{\mathcal D}}}
\def\ccalE{{\ensuremath{\mathcal E}}}

\def\ccalG{{\ensuremath{\mathcal G}}}

\def\ccalN{{\ensuremath{\mathcal N}}}

\def\ccalR{{\ensuremath{\mathcal R}}}
\def\ccalW{{\ensuremath{\mathcal W}}}

\def\ccalV{{\ensuremath{\mathcal V}}}

\def\ccal0{{\ensuremath{\mathcal 0}}}
\def\hhatd{{\ensuremath{\hat d}}}

\def\bbone{{\ensuremath{\mathbf 1}}}

\def\bbA{{\ensuremath{\mathbf A}}}
\def\bbB{{\ensuremath{\mathbf B}}}

\def\bbD{{\ensuremath{\mathbf D}}}
\def\bbE{{\ensuremath{\mathbf E}}}

\def\bbH{{\ensuremath{\mathbf H}}}
\def\bbI{{\ensuremath{\mathbf I}}}

\def\bbL{{\ensuremath{\mathbf L}}}

\def\bbW{{\ensuremath{\mathbf W}}}

\def\bbS{{\ensuremath{\mathbf S}}}

\def\bbX{{\ensuremath{\mathbf X}}}
\def\bbY{{\ensuremath{\mathbf Y}}}

\def\bba{{\ensuremath{\mathbf a}}}

\def\bbh{{\ensuremath{\mathbf h}}}

\def\bbx{{\ensuremath{\mathbf x}}}
\def\bby{{\ensuremath{\mathbf y}}}

\def\bb0{{\ensuremath{\mathbf 0}}}
\def\hbA{{\hat{\ensuremath{\mathbf A}} }}

\def\hbS{{\hat{\ensuremath{\mathbf S}} }}

\def\hbX{{\hat{\ensuremath{\mathbf X}} }}
\def\hbY{{\hat{\ensuremath{\mathbf Y}} }}

\def\tbA{{\tilde{\ensuremath{\mathbf A}} }}

\def\tbD{{\tilde{\ensuremath{\mathbf D}} }}

\def\bbDelta{\boldsymbol{\Delta}}
\def\bbTheta{\boldsymbol{\Theta}}

\def\bbPhi{\boldsymbol{\Phi}}

\usepackage{hyperref}
\hypersetup{    
colorlinks=true,
citecolor=blue(ryb),
linkcolor=blue(ryb),
filecolor=magenta,
urlcolor=cyan
}

\newcommand\revised[1]{#1}

\begin{document}
\title{Graph Convolutional Neural Networks Sensitivity under Probabilistic Error Model}
\author{Xinjue~Wang,~\IEEEmembership{Student Member,~IEEE,}
        Esa~Ollila,~\IEEEmembership{Senior Member,~IEEE,}
        and~Sergiy~A.~Vorobyov,~\IEEEmembership{Fellow,~IEEE}
\thanks{All the authors are with the Department of Information and Communications Engineering, Aalto University, Finland. This research was partially supported by the Research Council of Finland under Grant 359848 and~357715.}
}
\maketitle

\begin{abstract}
Graph Neural Networks (GNNs), particularly Graph Convolutional Neural Networks (GCNNs), have emerged as pivotal instruments in machine learning and signal processing for processing graph-structured data.
This paper proposes an analysis framework to investigate the sensitivity of GCNNs to probabilistic graph perturbations, directly impacting the graph shift operator (GSO).
Our study establishes tight expected GSO error bounds, which are explicitly linked to the error model parameters, and reveals a linear relationship between GSO perturbations and the resulting output differences at each layer of GCNNs.
This linearity demonstrates that a single-layer GCNN maintains stability under graph edge perturbations, provided that the GSO errors remain bounded, regardless of the perturbation scale.
For multilayer GCNNs, the dependency of system's output difference on GSO perturbations is shown to be a recursion of linearity.
Finally, we exemplify the framework with the Graph Isomorphism Network (GIN) and Simple Graph Convolution Network (SGCN).
Experiments validate our theoretical derivations and the effectiveness of our approach.
\end{abstract}
\begin{IEEEkeywords}
Sensitivity analysis, graph convolutional neural network, graph shift operator, structural perturbation
\end{IEEEkeywords}

\section{Introduction}
\label{sec:intro}
\IEEEPARstart{G}{raph} neural networks (GNNs) have steadily gained prominence as an innovative tool in machine learning and signal processing, exhibiting unparalleled efficiency in processing data encapsulated within complex graph structures~\cite{Bronstein17Geometric, Dong20-GSPML, Wu21Survey}. 
Uniquely designed, GNNs utilize a system of intricately coupled graph filters (GFs) with nonlinear activation functions, enabling the effective transformation and propagation of information within the graph~\cite{Isufi22GFOverview}.

Different GNN architectures can be delineated based on the GFs, which are an integral to the functioning of GNNs. 
A notable example of these architectures uses graph-convolutional filters.
The GNN employing this design is known as the Graph Convolutional Neural Network (GCNN).
Some examples of GCNNs include the vanilla Graph Convolutional Network (GCN)~\cite{Kipf17-vanillaGCN}, Graph Isomorphism Network (GIN)~\cite{Xu19-GIN}, Simple Graph Convolution Network (SGCN)~\cite{Li19-originalSGCN,Wu19-SGCN}, and Cayley Graph Convolutional Network (CayleyNet)~\cite{Levie19Cayley}.
In contrast to the aforementioned GCNNs, there exist non-convolutional GNNs such as the Graph Attention Network (GAT)~\cite{Velickovic18GAT} and Edge Varying Graph Neural Network (EdgeNet)~\cite{Isufi22EdgeNet}, which utilize edge-varying graph filters~\cite{Coutino19EdgeFilter}.

This paper delves into the GCNN, which blends graph convolutional filters with nonlinear activation functions. 
Graph convolutional filters couple the data and graph with the underlying graph matrix, named graph shift operator (GSO), which can be, for example, the graph adjacency matrix or graph Laplacian, encoding the interactions between data samples~\cite{Sandryhaila13-DSP}.
Based on the GSO, the graph filter captures the structural information by aggregating the data propagated within its $k-$hop neighborhoods, and feeds it to the next layer after processing, which can be applying graph coarsening and pooling~\cite{defferrard2016convolutional, Xu19-GIN}.
As the key component of GCNNs, GSO presents the graph structure, and is typically assumed to be perfectly known.
The precise estimation of the hidden graph structure is essential for successfully performing feature propagation in a convolution layer~\cite{Dong16-Learning, segarra2017network, buciulea2022learning}.

GSOs form the foundation of GCNN structures.
Any perturbation in the graph structure has a direct bearing on the operations of a GCNN.
Previous studies in graph signal processing (GSP) and GNN have examined both deterministic and probabilistic perturbations affecting GSOs.
A probabilistic graph perturbation model for a partially correct estimation of the adjacency matrix is proposed in~\cite{Miettinen21-ErrorEffect}, where a perturbed graph is modeled as a combination of the true adjacency matrix and a perturbation term specified by Erd\H{o}s-R\'enyi (ER) graph.
The work~\cite{Gao21-Stability} explores perturbations in graphs using random edge sampling, a scheme characterized by randomly deleting existing edges. 
In~\cite{xu19pgd}, a GSO perturbation strategy is formulated leveraging a general first-order optimization method, which concurrently imposes a constraint on the extent of edge perturbation.
In~\cite{Ceci20-Graph}, the authors propose to perturb eigenvector pairs of the graph Laplacian, considering single and multiple edge perturbations, under small perturbation assumption.
Here, small perturbations refer to changes in a small percentage of edges.

%
The stability of GFs and GCNNs under GSO perturbations is one of the key research areas in signal processing (SP) and computer science (CS).
In the SP community, research focuses on the relationship between the system's output differences and the GSO differences under evasion attacks, emphasizing changes in the learned representation.
In~\cite{Kenlay21-Rewire}, the authors provide bounds on the output changes of spectral GFs resulting from double edge rewiring on normalized augmented adjacency matrices. 
This study extends the stability results to SGCN and gives theoretical bounds.
In~\cite{Kenlay2021InterpretableSB}, the authors present interpretable bounds to verify the stability of spectral GFs against graph edge perturbations. 
These bounds are derived under the constraint that the degree of any node after perturbation cannot exceed twice its original degree.
In~\cite{Gama20-Stability}, the authors apply an additive error model with norm-bounded perturbations on unspecified GSOs to provide stability bounds for multi-layer GCNNs.
This model is not generic as it does not explicitly account for the perturbation of graph edges.
It primarily considers perturbations resembling a uniform scaling of edge weights, a limitation noted in~\cite{levie2021transferabilityJMLR}.
Additionally, the bound of error matrix is defined based on the smallest operator norm achievable via node permutation. 
However, this permutation assumption may not suit social or citation networks where node identification is label-dependent, as noted in~\cite{Kenlay21-Rewire}.
In~\cite{Gao21-Stability}, authors consider random edge deletions as the perturbation on GSOs, specifically focusing on adjacency matrices and graph Laplacians.
It concludes that both the GF and GCNN are linearly stable with respect to several factors, including the probability of edge dropping, nonlinearity, and the width and depth of the network architecture.
Nevertheless, in the experiments of~\cite{Gao21-Stability}, the maximum edge deletion probability is set to $6\%$, indicating a limited scale on perturbation.
Works in CS~\cite{dai18atkgph, Zugner18_AttackonGraph_KDD, wu19advegsgraphdata, wang2021certifiedKDD, lin22gphspecattck} focus on the effects of adversarial attacks affecting GCNN accuracy, considering both evasion and poisoning attacks.
The focus is on the impacts of such attacks on the downstream task.
For instance, under evasion attacks,~\cite{Zugner18_AttackonGraph_KDD} demonstrates the reduction on GCNN's accuracy under small perturbations, while maintaining the degree distributions after the attack, and~\cite{lin22gphspecattck} demonstrates the significant drop of accuracy of GCN when 5\% of edges are altered.


In this paper, we introduce a sensitivity analysis framework for GCNN under the probabilistic edge perturbation model~\cite{Miettinen21-ErrorEffect}.
\revised{We understand \textbf{\textit{stability}} as the characteristic of a system to maintain bounded output under perturbations, while \textbf{\textit{sensitivity analysis}} is an examination of how variations in the output depend on influencing factors.}
Our analysis concentrates on studying the effects of evasion attacks.
We use statistical analysis to give expected bounds for GSO errors (Theorem~\ref{thm:Case1Adj} and Proposition~\ref{prop:Case2NorAdj}).
These error bounds are explicitly dependent on the parameters of the error model.
Then, we establish a sensitivity analysis framework for both GF (Theorem~\ref{thm:gfdistance}) and multilayer GCNN (Theorem~\ref{thm:gcnsensitivity}) by giving expected bounds for differences of outputs because of GSO errors.
Finally, we exemplify the framework with GIN (Corollary~\ref{corr:GIN_MLP}) and SGCN (Corollary~\ref{corr:SGCN_sensitivity}), and empirically show that under large-scale graph perturbations (significant edge modifications), GCNNs maintain stability.

Our detailed contributions are summarized as follows.

\textit{1. Probabilistic error model.}
The probabilistic edge perturbation model considered is general and practically appealing.
It is grounded in stochastic block models, supports both deletion and addition of edges, and permits a broader perturbation scale. 
The corresponding analysis approach contrasts with the constrained perturbations in existing GCNN analyses, which involve such restrictions as permitting only edge deletions in~\cite{Gao21-Stability}, double edge rewiring in~\cite{Kenlay21-Rewire}, and small norm bounded errors in~\cite{Gama20-Stability}.

\textit{2. Tight GSO error bound.}
We give tighter expected bounds on GSO errors compared to our previous conference work \cite{Wang22-Eusipco}, in which the bounds are deterministic.
We use the $\ell_1$ norm suggested in \cite{Kenlay2021InterpretableSB} to bound the $\ell_2$ norm and make this bound interpretable by specifically tracking the changed node degrees, which can be directly linked to parameters of the error model (probabilities of deleting and adding edges).
Additionally, our bound does not require the eigendecomposition of GSO~\cite{Gama20-Stability, Gao21-Stability}, which is computationally heavy for large graphs.

\textit{3. Generic sensitivity analysis framework.}
Compared to previous works~\cite{Gama20-Stability, Gao21-Stability, Kenlay21-Rewire}, our proposed analysis framework is more generic in the following aspects. 
\textit{(i)} We remove the assumption on limited scale perturbation and allow for a large perturbation budget, for instance that 50\% of edges are deleted and 70\% of edges are added (compared to the original number of edges).
Our analysis is shown empirically to be valid even under such perturbation, while the maximum edge perturbation addressed in the current literature is $10\%$ of edges~\cite{Kenlay2021InterpretableSB}.
\textit{(ii)} We provide expected bounds under a probabilistic perspective, while the deterministic perturbations can be seen as special cases of our analysis.
\textit{(iii)} 
This framework is applicable to general GCNN models, with specific adjustments for GSO, graph shifts count, network layer count, and activation functions.

\textbf{Outline.}
The remainder of this paper is structured as follows.
In Sections~\ref{sec:preliminaries} and~\ref{sec:problem_formu}, we establish the fundamentals of GCNNs and proceed to formulate the problem.
Section~\ref{sec:GSOsensitivity} bounds the difference between original and perturbed GSOs, with particular emphasis on two cases: the adjacency matrix and its normalized version. 
Section~\ref{sec:GCNStability} encompasses both GFs and GCNNs like GIN and SGCN, and demonstrates that variations in the output of each GCNN layer in response to graph perturbations are linearly bounded.
Empirical validations presented in Section~\ref{sec:NumExperiments} use numerical experiments with both synthetic and real-world data to corroborate the proposed theorems, thereby attesting to the reliability of our sensitivity analysis model.
Section~\ref{sec:Conclusion} concludes the paper and discusses the future work.

%
\textbf{Notation.} 
Boldface lower case letters such as $\bbx$ represent column vectors, while boldface capital letters like $\bbX$ denote matrices. 
A vector full of ones is symbolized as $\ones_{N}$, and a $N \times N$ matrix full of ones is expressed as $\ones_{N \times N} = \ones_{N}\ones_{N}^\top$. 
The identity matrix of size $N \times N$ is represented as $\bbI_{N \times N}$.
The $i$-th row or column of the matrix $\bbA$ is given as $\bbA_i$, and the $(i,j)$-th element in matrix $\bbA$ is denoted as $[\bbA]_{i,j}$ or $\bbA_{i,j}$.
\revised{Vector $\ell_1$ norm is defined as follows: $\|\bba\|_1 = \sum_{j} |\bba_j|$.}
Matrix norms are defined as follows: the $\ell_1$ norm is represented as $\|\bbA\|_1 = \max_j\sum_i|\bbA_{i,j}|$, the $\ell_2$ norm as $\|\bbA\| = \|\bbA\|_2 = \sqrt{\max(\text{eig}(\bbA^\top\bbA))}$ (largest singular value of $\bbA$), and the $\ell_\infty$ norm as $\|\bbA\|_\infty = \max_i\sum_j|\bbA_{i,j}|$.
In addition, the Hadamard product is expressed with the symbol $\circ$.
We use $\textrm{Pr}(\cdot)$ for probability, $\mbE(\cdot)$ for expectation, $\Var(\cdot)$ for variance, and $\Cov(\cdot, \cdot)$ for covariance.

\section{Preliminaries}
\label{sec:preliminaries}
Graph theory, GSP, and GCNN form the cornerstone of data analysis in irregular domains. 
The GSO plays a key role in directing information flow across the graph, thereby enabling the creation of GFs and the design of GCNNs.

The sensitivity analysis of the GSO, which essentially involves matrix sensitivity analysis, provides an empirical insight into the system's resilience to perturbations. 
The GCNN, with its local architecture, maintains most of the properties of the graph convolutional filter, making it an ideal tool for sensitivity analysis. 
These preliminary concepts are essential for the implementation of sensitivity analysis in a graph-based context.

\noindent \textbf{Graph Basics.} 
Consider an undirected and unweighted graph~$\ccalG = (\ccalV, \ccalE, \ccalW)$, where the node set $\ccalV = \{1,\ldots,N\}$ consists of $N$ nodes, the edge set $\ccalE$ is a subset of $\ccalV \times \ccalV$, and the edge weighting function $\ccalW: \ccalV \times \ccalV \to \{ 0,1 \}$ assigns binary edges.
For an edge $(i, j) \in \ccalE$, we have $\ccalW(i, j) = \ccalW(j, i) = 1$ due to our focus on undirected and unweighted graphs. 
We define the $1$-hop neighboring set of a node $i$ as $\ccalN_i=\{j \in\ccalV:(i,j)\in\ccalE\}$, the degree of node $i$ as $d_i$, 
and the minimum degree of nodes around $i$ as $\tau_i = \min_{j \in \ccalN_i}d_j$. 

\noindent \textbf{GSO.}
The Graph Shift Operator (GSO) $\bbS \in \mbR^{N \times N}$ symbolizes the structure of a graph and guides the passage and fusion of signals between neighboring nodes.
It is often represented by the adjacency matrix $\bbA$, the Laplacian $\bbL$, or their normalized counterparts. 
These representations capture the graph's connectivity patterns, marking them indispensable tools for data analysis in both regular and irregular domains~\cite{Shuman13-EmergingGSP}.
The adjacency matrix, denoted by $\bbA$, incorporates both the weighting function and the graph topology $\ccalG$, where~$[\bbA]_{ij} = 1$ if $(i, j) \in \ccalE$ and $[\bbA]_{ij} = 0$ if $(i, j) \not \in \ccalE$.
The Laplacian matrix~$\bbL$ is defined by the adjacency matrix and a diagonal degree matrix $\bbD$. 
Specifically, $\bbL = \bbD - \bbA$, where~$\bbD = \diag(\bbA\bbone_N)$ is a diagonal matrix, and $[\bbD]_{ii} = d_i$. The value $d_i = \sum_{j\in\ccalN_i}[\bbA]_{ij}$ denotes the degree of node $i$.
Moreover, normalized versions of the adjacency and Laplacian matrices are defined as $\bbA_\textrm{n} = \bbD^{-1/2}\bbA\bbD^{-1/2}$ and $\bbL_\textrm{n} = \bbD^{-1/2}\bbL\bbD^{-1/2}$, respectively. 
These normalized versions help maintain consistency and manage potential variations in the scale of the data.

\noindent \textbf{Graph Convolutional Filter.} 
Using GSO, graph signals undergo shifting and averaging across their neighboring nodes. 
The signal on the graph is denoted by $\bbx \in \mbR^{N}$. 
Its $i$-th entry~$[\bbx]_i = x_i$ specifies the data value at the node $v_i$. 
The one time shift of graph signal is simply $\bbS \bbx$, whose value at node~$i$ is $[\bbS\bbx]_i = \sum_{j\in\ccalN_i}s_{ij}x_j$. 
After one graph shift, the value at node $i$ is given by moving a local linear operator over its neighborhood values $\{x_j\}_{j \in \ccalN}$. 
Based on the graph shifting, a graph convolutional filter $\bbh(\bbS)$ with $K$ taps is defined via polynomials of GSO and the filter weights $\bbh = \{h_k\}_{k=0}^{K}$ in the graph convolution 
\begin{equation}
    \label{eq:Smally}
    \bby = h_0\bbS^0\bbx + \cdots + h_{K}\bbS^{K}\bbx = \sum_{k=0}^Kh_k\bbS^k\bbx = \bbh(\bbS)\bbx,
\end{equation}
where $\bby$ is the filter's output and $\bbh(\bbS) = \sum_{k=0}^Kh_k\bbS^k$ is a shift-invariant graph filter with $K$ taps, and denotes the weight of local information after $K$-hop data exchanges.
The graph filter is then combined with the nonlinear activation function, forming the primary component of GCNN and contributing to its expressivity.

\noindent \textbf{Graph Perceptron and GCNN. }
A Graph Perceptron~\cite{Isufi22GFOverview} is a simple unit of transformation in the GCNN. 
The functionality of a graph perceptron can be seamlessly extended to accommodate graph signals with multiple features. 
Specifically, a multi-feature graph signal can be denoted by $\bbX = [\bbx_1, \cdots, \bbx_d] \in \reals^{N \times d}$, where $d$ signifies the number of features.
The architecture of an $L$-layer GCNN is built upon cascading multiple graph perceptrons. It operates such that the output of a graph perceptron in a preceding layer serves as the input to the graph perceptron at the subsequent layer $\ell$, where $\ell$ spans from $1$ to $L$.
We denote the feature fed to the first layer as $\bbX_0 = \bbX$.
For an $L$-layer GCNN, the graph perceptron at layer $\ell$ can be represented as
\begin{equation}
\label{eq:graphperceptron-multi}
\begin{split}
\bbY_{\ell} = \sum_{k=1}^K\bbS^k\bbX_{\ell-1}\bbH_{\ell k}, \ \
\bbX_{\ell} = \sigma_\ell \left( \bbY_{\ell} \right ).
\end{split}
\end{equation}
Here, $\bbY_{\ell}$ signifies the intermediate graph filter output, $\sigma_\ell(\cdot)$ denotes the nonlinear activation function at layer $\ell$,  and graph signals at each layer are $\bbX_{\ell}$ and $\bbX_{\ell-1}$ with sizes of $\reals^{N \times F_{\ell}}$ and $\reals^{N \times F_{\ell-1}}$, respectively, where $F_{\ell}$ denotes the number of features at the $\ell$-th layer. 
The bank of filter coefficients is represented by $\bbH = \{\bbH_{\ell k}\}_{\ell=1,\ldots,L;k=1,\ldots,K}$.
By recursively using \eqref{eq:graphperceptron-multi} until $\ell=L$, a general GCNN can be formulated as
\begin{align} \label{eq:finallayergcn}
\bbPhi(\bbX;\bbH,\bbS) = \bbX_L = \sigma(\revised{\sum_{k=1}^K} \bbS \bbX_{L-1}\bbH_{Lk}).
\end{align}
This representation captures the nature of GCNN operations, going through each layer and applying the corresponding transformation defined by the graph signal, filter coefficients, and the non-linearity function. 
This hierarchical arrangement facilitates the flow of information through successive layers, thus enabling effective learning from graph-structured data.

\section{Problem Formulation}
\label{sec:problem_formu}
A pivotal aspect of understanding the sensitivity of a GCNN is the considerations of potential alterations in the underlying graph structure. 
These alterations can be broadly construed as perturbations to the GSO, intrinsically linking to changes in the graph topology.
In the simplest form, any perturbation to the GSO can be depicted as
\begin{equation} \label{eq:generalGSOerror}
\hbS = \bbS + \bbE,
\end{equation}
where $\hbS$ signifies the perturbed GSO, $\bbS$ is the original GSO, and $\bbE$ represents the error term.
The spectral norm of this error term is denoted by
\begin{equation}
    \label{eq:gso_distance}
    d(\hbS, \bbS) = \|\hbS - \bbS\| = \| \bbE \|.
\end{equation}

Inspired by a previous work \cite{Miettinen21-ErrorEffect}, we utilize a probabilistic error model to represent graph perturbations, where each edge of the graph is subject to perturbation independently.
In this context, we primarily focus on the alterations occurring within the neighborhood of a particular node $u\in\ccalV$. 
More specifically, the perturbed neighborhood may encompass added nodes ($\ccalA_u$), deleted nodes ($\ccalD_u$), and remaining nodes ($\ccalR_u$), which ultimately leads to changes in node degree and modifications to the adjacency matrix.
We aim to quantify the sensitivity of GSO in relation to these perturbations. 
To this end, we adopt and expand upon the notation used in \cite{Kenlay21-Rewire, Kenlay2021InterpretableSB} for clarity and consistency. 

When the graph undergoes perturbations, it transforms into $\hat{\ccalG} = (\ccalV, \hat{\ccalE}, \hat{\ccalW})$, with the node set remaining unaffected.
We express degrees of node $u \in \ccalV$ in original and perturbed graphs as $d_u = \sum_j | [\bbA]_{u,j} |$ and $\hhatd_u = \sum_j|[\hbA]_{u,j}| = d_u + \delta_u$, respectively. 
Here, $\hbA$ denotes the adjacency matrix of the perturbed graph $\hat{\ccalG}$, and $\delta_u = \delta_u^+ - \delta_u^-$ is the degree change at node $u $, with $\delta_u^+ = |\ccalA_u|$ and $\delta_u^- = |\ccalD_u|$ corresponding to the number of edges added and deleted, respectively.
We will further delve into the assumptions for the error model and its effects on the GCNN's performance in the following discussion.

\subsection{Probabilistic Graph Error Model}
\begin{figure*}[t]
    \centering
    \subfloat[$\epsilon_1=0,\epsilon_2=0$]{
        \includegraphics[width = 0.24\linewidth, trim={1cm 0 1cm 0.8cm}, clip]{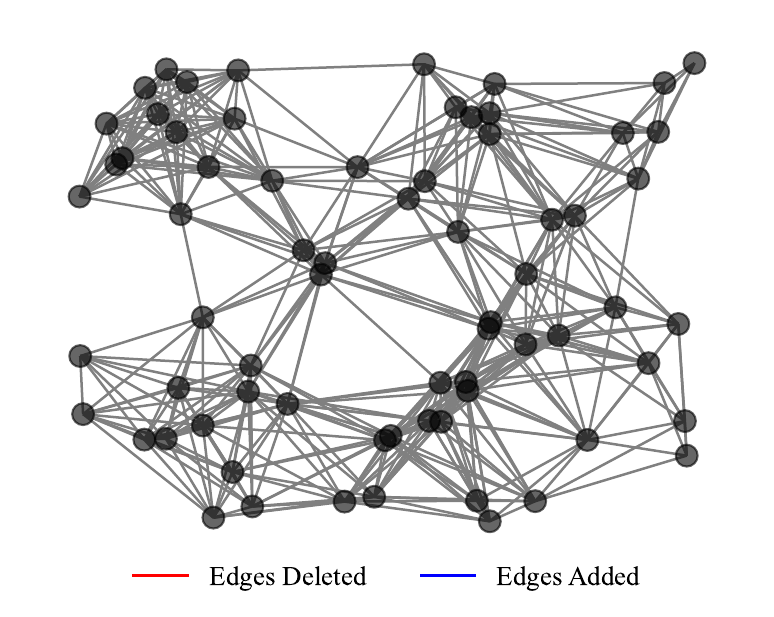}}
    \subfloat[$\epsilon_1=0.3,\epsilon_2=0$]{
        \includegraphics[width = 0.24\linewidth, trim={1cm 0 1cm 0.8cm}, clip]{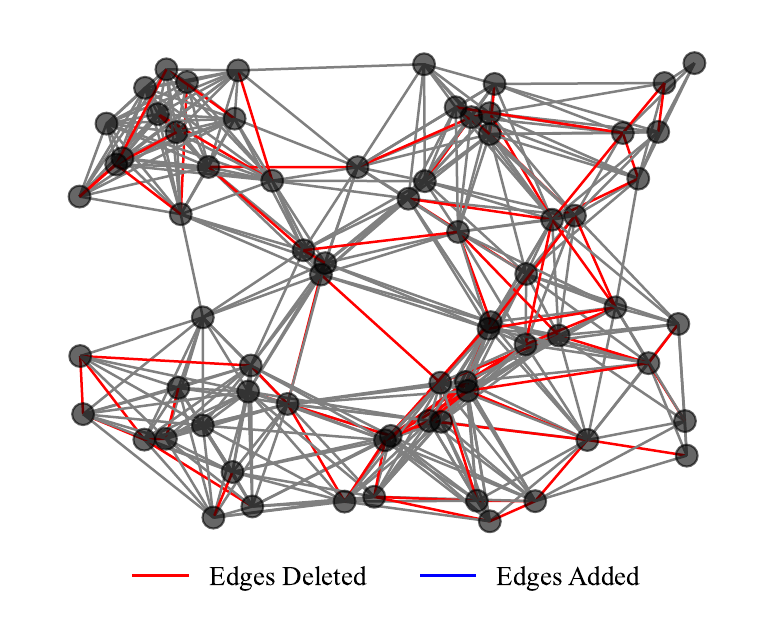}}
    \subfloat[$\epsilon_1=0,\epsilon_2=0.1$]{
        \includegraphics[width = 0.24\linewidth, trim={1cm 0 1cm 0.8cm}, clip]{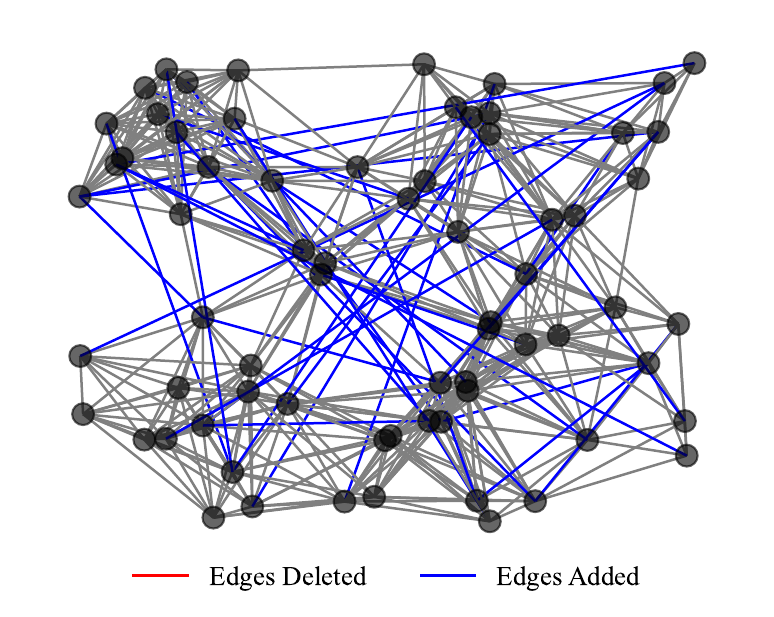}}
    \subfloat[$\epsilon_1=0.3,\epsilon_2=0.1$]{
        \includegraphics[width = 0.24\linewidth, trim={1cm 0 1cm 0.8cm}, clip]{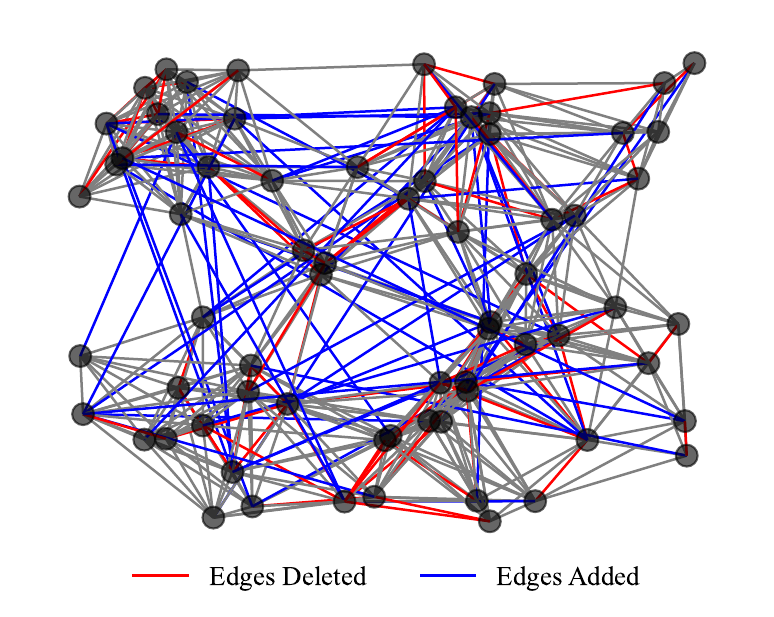}} 
    \caption{Visual representation of the probabilistic graph error model applied to a random geometric graph. 
    From left to right: (a) Original graph; (b) Graph after edge deletions ($\epsilon_1=0.3,\epsilon_2=0$); (c) Graph after edge additions ($\epsilon_1=0,\epsilon_2=0.1$); (d) Graph after both edge deletions and additions ($\epsilon_1=0.3,\epsilon_2=0.1$). 
    Deleted edges are marked in red and added edges are marked in blue.
    The transformations effectively illustrate the impact of perturbations modeled by \eqref{eq:basic_ermodel}.}
    \label{fig:ERmodelVis}
    \vspace{-3mm}
\end{figure*}

In this work, we utilize an Erd\"os-R\'enyi (ER) graph-based model for perturbations on a graph adjacency matrix, following the approach proposed in \cite{Miettinen21-ErrorEffect}. 
The adjacency matrix of an ER graph is characterized by a random $N \times N$ matrix $\mathbf{\Delta}_\epsilon$, where each element of the matrix is generated independently, satisfying $\textrm{Pr} ( [\mathbf{\Delta}_\epsilon]_{i,j} = 1 ) = \epsilon$ and $\textrm{Pr} ( [\mathbf{\Delta}_\epsilon]_{i,j} = 0 ) = 1 - \epsilon$ for all $i \neq j$. 
The diagonal elements are zero, i.e., $[\mathbf{\Delta}_\epsilon]_{i,i} = 0$ for $i = 1, \dots, N$, eliminating the possibility of self-loops.
For the sake of our analysis, we also assume that the perturbed graph $\hat{\ccalG}$ does not contain any isolated nodes, meaning that for all $u \in \ccalV$, $\hhatd_u \geq 1$.
The model can be adapted by employing the lower triangular matrix $\bbDelta_\epsilon^l$, and then defining $\bbDelta_\epsilon = \bbDelta_\epsilon^l + (\bbDelta_\epsilon^l)^\top$. 
Consequently, by specifying the error term in \eqref{eq:generalGSOerror}, the perturbed adjacency matrix of a graph signal can be expressed as
\begin{equation}\label{eq:basic_ermodel}
\hbA = \bbA - \bbDelta_{\epsilon_1}\circ\bbA + \bbDelta_{\epsilon_2}\circ(\ones_{N \times N} - \bbA),
\end{equation}
where the first term is responsible for edge deletion with probability $\epsilon_1$, and the second term accounts for edge addition with probability $\epsilon_2$. 
This error model can be conceptualized as superimposing two ER graphs on top of the original graph.
To better illustrate this model, we utilize visual aids based on a random geometric graph \cite{penrose2003random, hagberg2008networx}. 
Fig.~\ref{fig:ERmodelVis} visually represents the transition from the original graph to perturbed versions, which include the graph with only edge deletions ($\epsilon_1=0.3,\epsilon_2=0$), the graph with only edge additions ($\epsilon_1=0,\epsilon_2=0.1$), and the graph with both edge deletions and additions ($\epsilon_1=0.3,\epsilon_2=0.1$). 
Each state depicts the progressive impacts of the perturbations.

In this context, the impact of the perturbation on the degree of a given node $u \in \ccalV$ can be computed as follows. 
The effect of edge deletion is represented by $(-\bbDelta_{\epsilon_1} \circ \bbA)_u$, where each non-zero element in $\bbA_u$ has a probability of $\epsilon_1$ being deleted. 
Thus, the total number of deleted edges $\delta_u^-$ is the sum of $d_u$ independent and identically distributed (i.i.d.) Bernoulli random variables, each with a probability of $\epsilon_1$. 
Similarly, the effect of edge addition is denoted by $\left( \bbDelta_{\epsilon_2} \circ (\ones_{N \times N} - \bbA)\right)_u$, and the total number of added edges $\delta_u^+$ is the sum of $d_u^*$ i.i.d. Bernoulli random variables, each with a probability of $\epsilon_2$, where $d_u^*=N-d_u-1$.
Hence, we can express the number of deleted edges $\delta_u^-$ and the number of added edges $\delta_u^+$ as following binomial distributions:
\begin{equation}\label{eq:DeltauDistribution}
\begin{split}
\delta_u^- \sim \textrm{Bin}(d_u, \epsilon_1),\ \delta_u^+ \sim \textrm{Bin}(d_u^*, \epsilon_2),
\end{split}
\end{equation}
where $\textrm{Bin}(n, p)$ represents a binomial distribution with parameters $n$ and $p$.

\section{Expected Bound for GSO error}
\label{sec:GSOsensitivity}
\subsection{Error Bound for Unnormalized GSO Using $\ell_1$ Norm}
Building on the foundation laid by the discussion of graph structure perturbations and the proposed error model, we now outline the primary theoretical contributions of this study. 
Our focus here is to detail the probabilistic bounds that help quantify the sensitivity of the GSO to graph structure perturbations. 
We examine the case where the adjacency matrix serves as the GSO, implying $\hbS = \hbA$ and $\bbS = \bbA$.
The error model derived in \eqref{eq:basic_ermodel} can be expressed as
\begin{equation} \label{eq:basic_model}
    \bbE = \hbA - \bbA = - \bbDelta_{\epsilon_1}\circ \bbA + \bbDelta_{\epsilon_2}\circ (\ones_{N \times N} - \bbA).
\end{equation}
We can link the change in degree with the $\ell_1$ norm of error term in \eqref{eq:basic_model} as
\begin{equation}
    \|\bbE\|_1 = \max_{u \in \ccalV} \|\bbE_u\|_1,
\end{equation}
where 
\begin{equation}
    \revised{Y_u \triangleq \| \bbE_u \|_1  = |\ccalD_u| + |\ccalA_u| = \delta_u^- + \delta_u^+.}
\end{equation}
Let $Y \triangleq \max_{u \in \ccalV} Y_u$.
Since $\delta_u^-$ and $\delta_u^+$ are independent random variables, it is not appropriate to give deterministic upper bounds.
Instead, we present expected value bounds, which are better suited for analyzing the degree changes of nodes given the probabilistic nature of the model.
Our goal is to derive a closed-form expression for the expectation of the maximum node degree error, i.e.,
\begin{equation}\label{eq:expectationgeneralboundl2l1}
    \mbE[\|\bbE\|_1] = \mbE[\max_{u \in \ccalV} \|\bbE_u\|_1].
\end{equation}

The probability mass function (PMF) of $Y_u$ can be found by convolving the PMFs of $\delta_u^-$ and $\delta_u^+$, which are independent random variables.
Following binomial distributions in \eqref{eq:DeltauDistribution},
we can obtain the following PMFs
\begin{align}\label{eq:Delta_u^-and+PMF}
      \text{Pr}_{\delta_u^-}(k) & = \begin{pmatrix}d_u \\ k\end{pmatrix}\epsilon_1^k(1-\epsilon_1)^{d_u-k}, \ k=0,\ldots,d_u, \\
      \text{Pr}_{\delta_u^+}(k) & = \begin{pmatrix}d_u^*\\k\end{pmatrix}\epsilon_2^{k}(1-\epsilon_2)^{d_u^*-k},\ k = 0,\ldots,d_u^*,
\end{align}
where $d_u^* = N-d_u-1$, $\text{Pr}_{\delta_u^-}(k)$ and $\text{Pr}_{\delta_u^+}(k)$ represent the probabilities of $\delta_u^-$ and $\delta_u^+$ taking the value $k$, respectively.
Then, the PMF of $Y_u$ can be computed as
\begin{equation}
\begin{split}
    \text{Pr}_{Y_u}(k) &= \sum_{i=\max\{0, k-d_u^*\}}^{\min\{k, d_u\}} \text{Pr}_{\delta_u^-, \delta_u^+} (i, k-i)  \\ 
    &= \sum_{i=\max\{0, k-d_u^*\}}^{\min\{k, d_u\}} \text{Pr}_{\delta_u^-}(i)\text{Pr}_{\delta_u^+}(k-i),
    \label{eq:Y_uPMF}
\end{split}
\end{equation}
where $k=0,\ldots,N-1$.
Using \eqref{eq:Y_uPMF}, the cumulative distribution function (CDF) of $Y$ is computed as
\begin{align}
\text{F}_{Y}(k) & =  \text{Pr}(Y \leq k) =  \text{Pr}(\max(Y_1,\ldots,Y_N) \leq k) \notag \\
& =  \text{Pr}(Y_1 \leq k, \ldots, Y_N \leq k) = \prod_{u=1}^{N}  \text{Pr}(Y_u \leq k).
\end{align}
Given that $Y_u$ for $u\in\ccalV$ are i.i.d. and for $k=1,\ldots,N-1$, the CDFs for $Y$ and $Y_u$ are as follows
\begin{align} \label{eq:Y_CDF}
    \text{F}_{Y}(k) = \prod_{u=1}^{N} \text{F}_{Y_u}(k),\quad \text{F}_{Y_u}(k) = \sum_{j=0}^{k} \text{Pr}_{Y_u}(j).
\end{align}
With the PMF of $Y$ taking on a specific value $k$ being $\text{Pr}_{Y}(k) = \text{F}_{Y}(k) - \text{F}_{Y}(k-1)$, the expectation of $Y$ can be represented as
\begin{align}\label{eq:Y_Exp}
\mbE[Y] & = \sum_{k=1}^{N-1}k\text{Pr}_{Y}(k) = \sum_{k=1}^{N-1}k \left[ \text{F}_{Y}(k) - \text{F}_{Y}(k-1) \right],
\end{align}
which provides a closed-form expression for $\mbE[Y] = \mbE[\|\bbE\|_1]$.
The variance of $Y$ can also be given as 
\begin{equation}
    \Var[Y] = \Var[\|\bbE\|_1] = \mbE[Y^2] - (\mbE[Y])^2,
    \label{eq:Y_Var}
\end{equation}
where $\mbE[Y^2] = \sum_{k=1}^{N-1}k^2\text{Pr}_{Y}(k) $.

\subsection{Bridging $\ell_1$ and $\ell_2$ Norms in GSO Analysis}
In the analysis of graph-structured data, the spectral norm ($\ell_2$ norm), is often employed to quantify the graph spectral error. 
While \cite{Wang22-Eusipco} did furnish a spectral error bound for the GSO, the need for a more refined and interpretable bound persists to enable more comprehensive analyses.
Following the approach of \cite{Kenlay2021InterpretableSB}, this study uses the $\ell_1$ norm and assumes that the error matrix $\bbE$ is fixed.
The proposed approach of bounding $\|\bbE\|$ is based on assumptions of an undirected graph and perturbation $\bbE = \bbE^\top$. 
Using inequalities $\|\bbE\|^2 \leq \|\bbE\|_1 \|\bbE\|_\infty$ \cite[Section 2.3.3]{golub2012matrix} and the fact that in our case $\|\bbE\|_1  = \|\bbE\|_\infty$, the $\ell_2$ norm can be bounded by the $\ell_1$ norm
\begin{equation}\label{eq:l2smallerthanl1}
    \|\bbE\| \leq \|\bbE\|_1 = \max_{u \in \ccalV} \|\bbE_u\|_1.    
\end{equation}
The entries in the error matrix $\bbE$ of equation \eqref{eq:basic_model} are random variables. 
As such, it is challenging to derive a deterministic bound for \eqref{eq:l2smallerthanl1} that is both tight and generalizable. 
In contrast, an expected bound 
\begin{equation} \label{eq:expectationgeneralboundl2l1}
    \mbE[ \|\bbE\| ] \leq \mbE[ \|\bbE\|_1 ] = \mbE[ \max_{u \in \ccalV} \|\bbE_u\|_1],   
\end{equation}
provides a more reasonable estimate of the true behavior of the error matrix, as it takes into account the distribution of the random variables, as well as the structural changes of the perturbed graph.
Thus, we have the following theorem.
\begin{theorem}\label{thm:Case1Adj}
In the context of the probabilistic error model \eqref{eq:basic_model}, let GSO be adjacency matrix $\bbS = \bbA$, and perturbed GSO be $\hbS=\hbA$, then, a closed-form expression for the upper bound on the expectation of the GSO distance is given by
\begin{equation} \label{eq:l1boundthm1}
    \mbE\left[ d(\hbS,\bbS) \right] \leq \mbE[Y],
\end{equation}
where $\mbE[Y]$ is computed using \eqref{eq:Y_Exp}, \eqref{eq:Y_CDF}, and \eqref{eq:Y_uPMF}.
\end{theorem}
Theorem \ref{thm:Case1Adj} provides a closed-form expression for the upper bound, which are explicitly dependent on the parameters $(\epsilon_1, \epsilon_2)$ of the probabilistic error model in \eqref{eq:basic_model}.
Using a loose upper bound proposed in~\cite{aven1985upper},
we can bound \eqref{eq:l1boundthm1} as
\begin{align}
     \mbE[Y] & \leq \max_{1\leq u\leq N} (d_u\epsilon_1 + d_u^*\epsilon_2) \nonumber \\
     & + \sqrt{\frac{N-1
     }{N}\sum_{u=1}^N \big(d_u\epsilon_1(1-\epsilon_1) + d_u^* \epsilon_2(1-\epsilon_2)\big) }. \label{eq:bound4l1bound}
 \end{align}
We note that \eqref{eq:bound4l1bound} showcases how our bound in Theorem~\ref{thm:Case1Adj} is parameterized by the probabilities of adding and deleting edges.
Thus, Theorem~\ref{thm:Case1Adj} precisely captures the resulting structural changes induced by the probabilistic error model, unlike the generic spectral bound in \cite{Wang22-Eusipco}, which overlooks specific structural changes on the perturbed GSO.
\begin{remark}[Why not use $\ell_2$ norm?]
The spectral bounds derived using the $\ell_2$ norm, as presented in \cite{Wang22-Eusipco}, cannot fully capture the specific structural changes to the GSO from perturbations, especially in graphs with unique properties like degree distribution or sparsity. 
Focused on worst-case scenarios, these bounds lead to overestimations, rendering them looser and less applicable to particular graph types.
The $\ell_1$ norm is preferred over the $\ell_2$ norm for providing an upper bound because it reveals the impact of structural changes denoted by $\bbDelta_{\epsilon_1}$ and $\bbDelta_{\epsilon_2}$ in~\eqref{eq:basic_model}, whereas the $\ell_2$ norm absorbs these structural changes into the overall spectral change, making it more challenging to derive a tight bound.
\end{remark}

\subsection{Error Bound for Normalized GSO}
In this context, the GSO is considered as the normalized version of the adjacency matrix, i.e., $\bbS = \bbA_\textrm{n}$. 
The entries of the normalized adjacency matrix are as follows, $[\bbA_\textrm{n}]_{u,v} = \frac{1}{\sqrt{d_ud_v}}$ if $(u,v)\in\ccalE$, and $[\bbA_\textrm{n}]_{u,v} = 0$ if $(u,v)\not\in\ccalE$.
In \cite{Kenlay2021InterpretableSB}, a closed form for $\|\bbE_u\|_1$ is proposed
\begin{equation}\label{eq:KenlayGeneralBound}
\begin{split}
\|\bbE_u\|_1 =  \sum_{v \in \ccalD_u}\dfrac{1}{\sqrt{d_ud_v}} + \sum_{v \in \ccalA_u}\dfrac{1}{\sqrt{\hhatd_u\hhatd_v}} \\ 
+ \sum_{v \in \ccalR_u}\left|\dfrac{1}{\sqrt{d_ud_v}} - \dfrac{1}{\sqrt{\hhatd_u\hhatd_v}}\right|,
\end{split}
\end{equation}
where $\hhatd_u$ and $\hhatd_v$ denote the degrees of node $u$ and $v$ after perturbation.
However, the assumption in \cite{Kenlay2021InterpretableSB} states that the degree alteration $\hhatd_v$ should not exceed twice the initial degree, i.e., $\hhatd_v\leq 2d_v, v \in \{ \ccalN_u \cup {u} \}$.
This restriction is not needed in our work.
Following the error model in \eqref{eq:basic_ermodel}, this limitation could easily be breached with an increased probability of edge addition $\epsilon_2$.
We start with the following lemma.
\begin{lemma}
\label{lmm:EuNormalized}
    Let $\bbE_u$ be defined as in \eqref{eq:KenlayGeneralBound}, then its $\ell_1$ norm is bounded by a random variable $Z_u$
    \begin{align}\label{eq:Z_u}
        \|\bbE_u\|_1 \leq Z_u = Z_{u,1} +  Z_{u,2},
    \end{align}
    where $Z_u$ is defined as the sum of $Z_{u,1}=\sqrt{d_u/\tau_u}$ and $Z_{u,2}=\sum_{v \in \ccalA_u\cup\ccalR_u}\frac{1}{\sqrt{(d_u+\delta_u^+ -\delta_u^-)(d_v + \delta_v^+ -\delta_v^-)}}$, $d_u$ is the degree of node $u$, $\tau_u$ is the minimum degree of neighboring nodes of $u$, and $\delta_u^-, \delta_u^+, \delta_v^-, \delta_v^+$ are random variables with binomial distributions as $\delta_u^- \sim \textnormal{Bin}(d_u, \epsilon_1), \delta_u^+ \sim \textnormal{Bin}(d_u^*, \epsilon_2), \delta_v^- \sim \textnormal{Bin}(d_v, \epsilon_1), \delta_v^+ \sim \textnormal{Bin}(d_v^*, \epsilon_2)$ for $u\in\ccalV$ and $v \in \ccalA_u\cup\ccalR_ u$, where $d_u^* = N-d_u-1$ and $d_v^* = N-d_v-1$.
\end{lemma}
\begin{proof}
    See Appendix \ref{lmm:EuNormalized_proof}.
\end{proof}

Let 
\begin{align}\label{eq:Z_definition}
Z \triangleq \max_{u\in\ccalV}Z_u,    
\end{align}
and note that $Z_u$ and $Z$ are discrete random variables.
While the binomial random variables and degrees in the expression for $Z$ are assumed to be i.i.d., the inherent nonlinearity and high-dimensionality in the function, along with the complexity introduced by the maximization operation over all nodes, pose challenges for deriving an analytical expression for $\mbE[Z]$. 
Furthermore, the expectation of a maximum of random variables often lacks a simple closed form with only bounds often being derivable, not the exact value.
On the other hand, Monte Carlo simulations provide an efficient alternative for estimating $\mbE[Z]$, which is given as
\begin{equation} \label{eq:ZMCexpectation}
    \mu_Z \triangleq \mathbb{E}[Z]  \approx  \frac{1}{N_{\textrm{samp}}} \sum_{i=1}^{N_{\textrm{samp}}}Z_{(i)} = \hat \mu_Z,
\end{equation}
where $Z_{(i)}$ represents the outcome from the $i$-th Monte Carlo trial.
Thus, for the normalized GSO, we have the following proposition as the counterpart of Theorem~\ref{thm:Case1Adj}.
\begin{proposition}\label{prop:Case2NorAdj}
    In the context of the probabilistic error model \eqref{eq:basic_model}, let GSO be normalized adjacency matrix $\bbS = \bbA_\textrm{n}$, and perturbed GSO being $\hbS=\hbA_\textrm{n}$.
    Then, an upper bound on the expectation of the GSO distance is given by
\begin{equation}
    \mbE\left[ d(\hbS,\bbS) \right] \leq \mbE[Z], 
    \label{eq:oldthm2;case2adj}
\end{equation}
where $\mbE[Z]$ is computed using \eqref{eq:ZMCexpectation}, \eqref{eq:Z_definition}, and Lemma \ref{lmm:EuNormalized}.
\end{proposition}
The upperbound provided in Proposition~\ref{prop:Case2NorAdj} focuses specifically on normalized adjacency matrices. 
This result complements the analysis for the unnormalized case.
\revised{
We note that the bound for normalized GSO is not an approximation or an empirical estimation; it presents a theoretical upperbound.
The only difference between the bound in Proposition~\ref{prop:Case2NorAdj} and the bound in Theorem~\ref{thm:Case1Adj} is the computation.
As for the bound in Theorem~\ref{thm:Case1Adj} (unnormalized case), $\mbE[Y]$ has a closed-form expression; while for computing the bound in Proposition~\ref{prop:Case2NorAdj} (normalized case) $\mbE[Z]$, we use Monte Carlo simulations.
}

\section{GCNN Sensitivity}
\label{sec:GCNStability}
\subsection{Graph Filter Sensitivity Analysis}
\label{subsec:_gfsenanalysis}
The sensitivity of graph filters is a critical aspect that follows logically from the preceding discussion on the expected bounds of GSO errors. 
Having extensively delved into the properties of GSO perturbations, we now turn our attention to the graph filters. 
Graph filters, being polynomials of GSOs, inherit the perturbations in the graph structure, manifesting as variations in filter responses. 

The sensitivity of a graph filter to perturbations in the GSO is captured by the theorem below, which establishes a bound on the error in the graph filter response due to perturbations in the GSO and the filter coefficients.
\begin{theorem}[Graph filter sensitivity]
    \label{thm:gfdistance}
    Let $\bbS$ and $\hbS$ be the GSO for the true graph $\ccalG$ and the perturbed graph $\hat{\ccalG}$, respectively. 
    \revised{The distance between polynomial graph filters} 
    $\bbh(\bbS) = \sum_{k=0}^{K}h_k\bbS^k$ and $\bbh(\hbS) = \sum_{k=0}^{K}h_k\hbS^k $ is defined as 
    \begin{equation}\label{eq:filter_distance}
        d \big( \bbh(\hbS), \bbh(\bbS) \big) = \| \bbh(\hbS) - \bbh(\bbS) \|.
    \end{equation}
    The expectation of filter distance \eqref{eq:filter_distance} is bounded as
    \begin{equation}\label{eq:filterstability}
        \mbE \left[ d \big( \bbh(\hbS), \bbh(\bbS) \big) \right] 
        \leq 
        \sum_{k=1}^{K} k  |h_k| \left( \lambda_{k} \mathbb{E}[\|\mathbf{E}\|] + \zeta_k \right),
    \end{equation}
    where $\lambda_{k} \triangleq \mbE[\lambda^{k-1}]$, $\zeta_k \triangleq \Cov[\|\bbE\|,\lambda^{k-1}]$, and $\lambda = {\max} \{ \| \hbS \|, \| \bbS \| \}$ denotes the largest of the maximum singular values of two GSOs.
\end{theorem}
\begin{proof}
See Appendix \ref{apdx:gfdistance_proof}.
\end{proof}
%
\revised{Theorem \ref{thm:gfdistance} reveals that the expected graph filter distance is linearly bounded by the expected GSO distance, $\mbE \left[ \|\bbE \| \right]$, if the sufficient condition $\lambda = \|\bbS\|$ is met.}
This bound is influenced by: the filter degree $K$, the maximum singular value $\lambda$ of GSOs, and the filter coefficients $\{h_k\}_{k=1}^K$. 
The theorem indicates that higher order graph filters are likely to exhibit greater instability.
In Section~\ref{sub,sec:gf_sensitivity_test}, we present a supporting experiment, specifically for low-pass graph filters with the unnormalized GSO, $\bbS = \bbA$.

\subsection{GCNN Sensitivity Analysis}
\label{subsec:_gcnnsenanalysis}
Based on the sensitivity analysis of graph filter, we extend this study to the sensitivity analysis of the general GCNN.
Instead of meticulously quantifying the specifics of each perturbed graph, we propose a probabilistic boundary that captures the potential magnitude of graph perturbations and more insightful assessment of the system's sensitivity to graph perturbations.
We present the following theorem to exemplify this approach, encapsulating the sensitivity of a general GCNN to GSO perturbations.
\begin{theorem}[GCNN Sensitivity] \label{thm:gcnsensitivity}
For a general GCNN under the probabilistic error model \eqref{eq:basic_model}, the expected difference of outputs at the final layer $L$ is given as
\begin{align}
\mbE\left[ \left\| \hbX_{L} - \bbX_{L} \right\| \right] 
\leq C_{\sigma_L} B_L \mbE\left[\|\bbE\| \right] + C_{\sigma_L}D_L,
\end{align}
where $C_{\sigma_\ell}$ represents the Lipschitz constant for the nonlinear activation function used at layer $\ell$, for $\ell=1,\ldots,L$, $B_\ell$ and $D_\ell$ for $\ell=1$ and then for $\ell=2,\ldots,L$ are defined as follows
\begin{equation}\label{eq:thm3_bounds}
    \begin{split}
        & B_1 = \sum_{k=1}^K k \lambda_{k} \|\bbX_0\| \|\bbH_{1k}\|,
        D_1 = \sum_{k=1}^K k\zeta_k \|\bbX_0\| \|\bbH_{1k}\|, \\
        & B_\ell = \sum_{k=1}^K \left( \lambda_{k+1}  C_{\sigma_{\ell-1}} B_{\ell-1} + k\lambda_k\|\bbX_{\ell-1}\| \right) \|\bbH_{\ell k}\|, \\
        & D_\ell = \sum_{k=1}^K \left( \mu_{k,\ell-1} + \lambda_k C_{\sigma_{\ell-1}} D_{\ell-1} + k \zeta_k\|\bbX_{\ell-1}\| \right) \|\bbH_{\ell k}\|,
    \end{split}
\end{equation}
with constant $\mu_{k,\ell-1} \triangleq \sqrt{\Var[\|\hbX_{\ell-1} - \bbX_{\ell-1}\|]\Var[\lambda^k]}$, and $\lambda_{k}$ and $\zeta_k$ in Theorem~\ref{thm:gfdistance}, for $k=1,\ldots,K$. 
\end{theorem}
\begin{proof}
    See Appendix \ref{apdx:gcnsensitivity_proof}.
\end{proof}
\noindent In Theorem~\ref{thm:gcnsensitivity}, we use recursive bounds containing inter-layer features to simplify the formulation.
Note that these inter-layer features $\{\bbX_{\ell-1}, \hbX_{\ell-1}\}_{\ell=2}^L$ 
can be explicitly computed by the initial input feature $\bbX_0$, both original and perturbed GSOs $(\bbS, \hbS)$, GCNN parameters (number of layers $L$ and graph shift $K$, network's learned weights $\{\bbH_{\ell k}\}$, and activation functions $\sigma(\cdot)$).
The derivation process employs induction.
For the first layer $\ell=1$, we have $\bbX_1 = \sigma_1(\sum_{k=1}^K\bbS^k\bbX_0\bbH_{1k})$ and $\hbX_1 = \sigma_1(\sum_{k=1}^K\hbS^k\bbX_0\bbH_{1k})$;
for the second layer $\ell=2$, the features are $\bbX_{2} = \sigma_2 ( \sum_{k=1}^K\bbS^k\bbX_1\bbH_{2 k})$ and $\hbX_{2} = \sigma_2 (\sum_{k=1}^K\hbS^k\hbX_1\bbH_{2k})$; 
by induction, for the $\ell-1$th layer, we have
\begin{equation}
    \begin{split}
        \bbX_{\ell-1} & = \sigma_\ell \left(\sum_{k=1}^K\bbS^k\bbX_{\ell-2}\bbH_{\ell-1, k}\right), \\
        \hbX_{\ell-1} & = \sigma_\ell \left(\sum_{k=1}^K\hbS^k\hbX_{\ell-2}\bbH_{\ell-1, k}\right).
    \end{split}
\end{equation}

Theorem~\ref{thm:gcnsensitivity} forms the bedrock of our analysis, quantifying how GCNNs respond to graph perturbations, which is described by a linear relationship at each layer.
The sensitivity of multilayer GCNN to perturbations can be represented by a recursion of linearity.
For multilayer GCNN, its expected output difference is controlled by: \textit{(i)} the input feature, \textit{(ii)} the GSO, error model parameters, \textit{(iii)} Lipschitz constants of activation functions, and \textit{(iv)} GCNN weights.
We note that, choosing activation functions with more conservative Lipschitz constants can possibly improve the stability of GCNNs by imposing more constraints on the recursion.
However, this may suppress the performance of a neural network, as noted in~\cite{ohayon2023perceptionrobustness}.
Our sensitivity analysis framework is generic, allowing for simplifications such as assuming a unit Lipschitz constant and normalized input features, as suggested in~\cite{Kenlay21-Rewire}.
However, these simplifications do not indicate that the GCNN sensitivity is unaffected by the Lipschitz constant or input features.
This layered analysis also enables an understanding of how perturbations propagate through GCNN layers, impacting the overall performance.
Additionally, Theorem~\ref{thm:gcnsensitivity} does not restrict the scale of graph perturbations, which is a typical restriction in the existing literature.

Within the evasion attack context, where the focus is on learned representations, we demonstrate the following property: given that the GSO error is bounded as in Theorem~\ref{thm:Case1Adj} and Proposition~\ref{prop:Case2NorAdj}, the linear bound of each layer of GCNN (illustrated in Subsection~\ref{sub,sub:GCNN_Linearity_Corroboration}) permits the network's stability against perturbation as long as the graph error remains within the bound.
In Subsection~\ref{sub,sub:GCNN_Accuracy_Corroboration}, we show that multilayer GCNN is stable by showing its finite responses to large scale perturbations, even under notable declines in accuracy.

\subsection{Specifications for GCNN variants}
Building upon sensitivity analysis Theorem \ref{thm:gcnsensitivity}, our discussion now evolves towards two specific GCNN variants - GIN~\cite{Xu19-GIN} and SGCN~\cite{Li19-originalSGCN, Wu19-SGCN}.
\revised{They apply different GSOs for feature propagation.}
In GIN, the GSO for each layer is chosen as a partially augmented unnormalized adjacency matrix; in SGCN, the GSO is chosen as a normalized augmented adjacency matrix.
This choice is made to align with the discussions on tight GSO bounds in Section \ref{sec:GSOsensitivity}.
By focusing on GIN and SGCN, we are essentially extending our theoretical understanding to practical and real-world applications.

\subsubsection{Specification for GIN}
\label{sec,subsub:GIN}
The GIN is designed to capture the node features and the graph structure simultaneously. 
The primary intuition behind GIN is to learn a function of the feature information from both the target node and its neighbors, which is related to the Weisfeiler-Lehman (WL) graph isomorphism test \cite{Weisfeiler68-WLTest}. 
The chosen GSO for GIN is $\bbS = \bbA + (1+\varepsilon)\bbI$, where the learnable parameter $\varepsilon$ preserves the distinction between nodes in the graph that are connected differently, and prevents GIN from reducing to a WL isomorphism test. 

Given the GSO above, only the first order term with $K=1$ in \eqref{eq:Smally} is kept, and the intermediate output of such graph filter is $\bby = \bbS\bbx$.
A node Multilayer Perceptron (MLP) $\bbh_{\bbTheta}$ is then applied to the filter's output as $\bbh_{\bbTheta}(\bby)$.
Assuming the inner MLP has two layers in each GIN layer, a single-layer GIN ($L=1$) can be represented as
 \begin{align}\label{eq:GIN_singlelayer}
    \bbX_{L} = \sigma_{L 2} ( \sigma_{L 1} ( \bbS \bbX_{L-1} \bbW_{L1} + \bbB_{L1} )\bbW_{L 2} + \bbB_{L2} ),
\end{align}
where $\left(\bbW_{L 1}, \bbB_{L1}, \sigma_{L 1}(\cdot)\right)$ are weight matrix, bias matrix, and nonlinearity function in the first layer of the MLP, and $\left(\bbW_{L 2}, \bbB_{L2}, \sigma_{L2}(\cdot)\right)$ are weight matrix, bias matrix, and nonlinearity function in the second layer of the MLP.
Then, we provide the following corollary.
\begin{corollary}[The sensitivity of single-layer GIN]\label{corr:GIN_MLP}
    For the single-layer GIN \revised{($L=1$)} in~\eqref{eq:GIN_singlelayer} under the probabilistic error model~\eqref{eq:basic_model}, the expected difference of outputs because of GSO perturbations is given as 
    \begin{align}
        \mbE \left[\| \hbX_{L} - \bbX_{L} \|\right] \leq \xi \mbE\left[\| \bbE \|\right],
    \end{align}
    with constant
    \begin{align}
        \xi = C_{\sigma_{L 2}} C_{\sigma_{L 1}}  \|\bbW_{L2}\| \|\bbW_{L1}\| \|\bbX_{L-1}\|,
    \end{align}
    where $\bbX_{L-1} = \bbX_0$ is the input feature.
\end{corollary}
\begin{proof}
    See Appendix~\ref{apdx:ginpercepsensitivity_proof}.
\end{proof}
Corollary \ref{corr:GIN_MLP} shows a linear dependency between the output difference of a single-layer GIN and GSO perturbations.
In GIN, node vector transformations by MLP contribute significantly to network's expressivity.
\revised{Under evasion attacks, with Corollary~\ref{corr:GIN_MLP}, the analysis of these transformed node representations is straightforward.}

%
\subsubsection{Specification for SGCN}
\label{sec,subsub:SGCN}
The SGCN is a streamlined model, developed by aiming to simplify a multilayered GCNN through the utilization of an affine approximation of graph convolution filter and the elimination of intermediate layer activation functions.
The GSO chosen for SGCN is $\bbS = \tbD^{-1/2}\tbA\tbD^{-1/2}$, where $\tbA = \bbA + \bbI$ is the augmented adjacency matrix and $\tbD$ is the corresponding degree matrix of the augmented graph.

\revised{Given the normalized augmented GSO, the node degrees~$d_u, u=1,\ldots,N$ are redefined based on the augmented GSO,~specifically, they are incremented by $1$ compared to their values in the non-augmented version.}
This streamlined model simplifies the structure of a vanilla GCN \cite{Kipf17-vanillaGCN} by retaining a single layer and the $K$th order GSO in \eqref{eq:Smally}, so the output of the filter is $ \bby = h_{K}\bbS^{K}\bbx$.
Note that for a SGCN, the maximum number of layers is $L=1$.
Consequently, the output of a single-layer SGCN using a linear logistic regression layer is represented as
\begin{equation}\label{eq:SGCN_singlelayer}
    \bbX_L = \sigma_L (\bbS^K \bbX \bbH_K),
\end{equation}
and thus, we can easily give the following corollary.
\begin{corollary}[The sensitivity of SGCN]\label{corr:SGCN_sensitivity}
    For the SGCN in~\eqref{eq:SGCN_singlelayer} under the probabilistic error model~\eqref{eq:basic_model}, the expected difference of outputs because of GSO perturbations is given as 
    \begin{align}
        \mbE\left[ \| \hbX_L - \bbX_L \| \right] \leq C_{\sigma_L} B_L \mbE\left[\|\bbE\|\right] + C_{\sigma_L} D_L,
    \end{align}
    where $B_L = \lambda_{K} \|\bbX\| \|\bbH_{K}\|$, $D_L = K\zeta_K \|\bbX\| \|\bbH_{K}\|$, $\lambda_{K}$ and $\zeta_K $ are defined in Theorem~\ref{thm:gcnsensitivity}.
\end{corollary}
With Corollary \ref{corr:SGCN_sensitivity}, we conclude that the sensitivity analysis for SGCN is a specification for the general form of a multilayer GCNN.

\section{Numerical Experiments}
\label{sec:NumExperiments}
\begin{figure}[!t]
    \centering 
    \includegraphics[width = 0.98\linewidth]{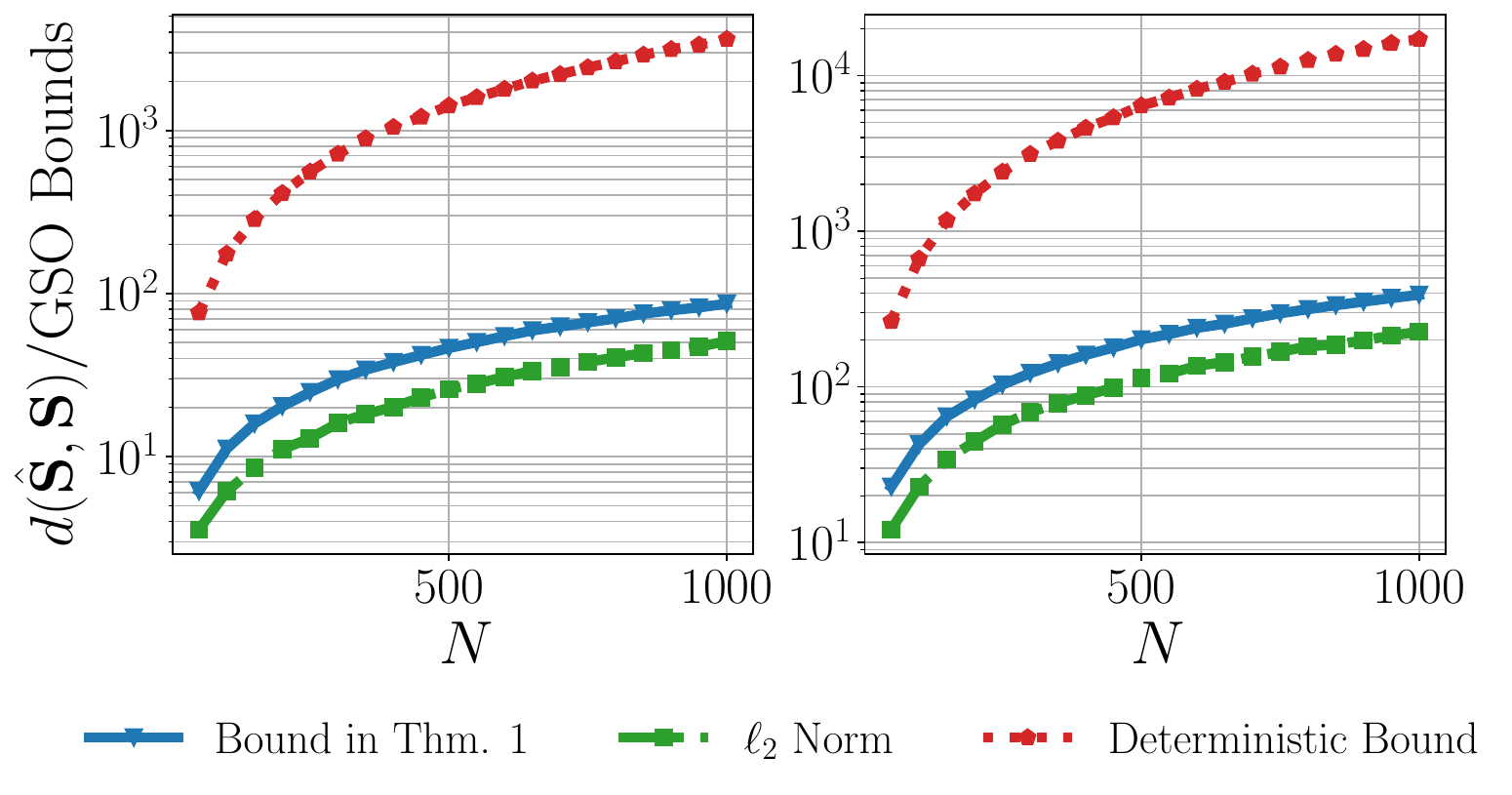}
    \caption{Comparative analysis of our bound in Theorem~\ref{thm:Case1Adj}, the deterministic bound in Theorem 2 of~\cite{Wang22-Eusipco}, and the empirical GSO distance in $\ell_2$ norm.}
    \label{fig:theo_bdcrr_case0}
\end{figure}

\begin{figure}[!t]
    \centering
    \includegraphics[width = 0.98\linewidth]{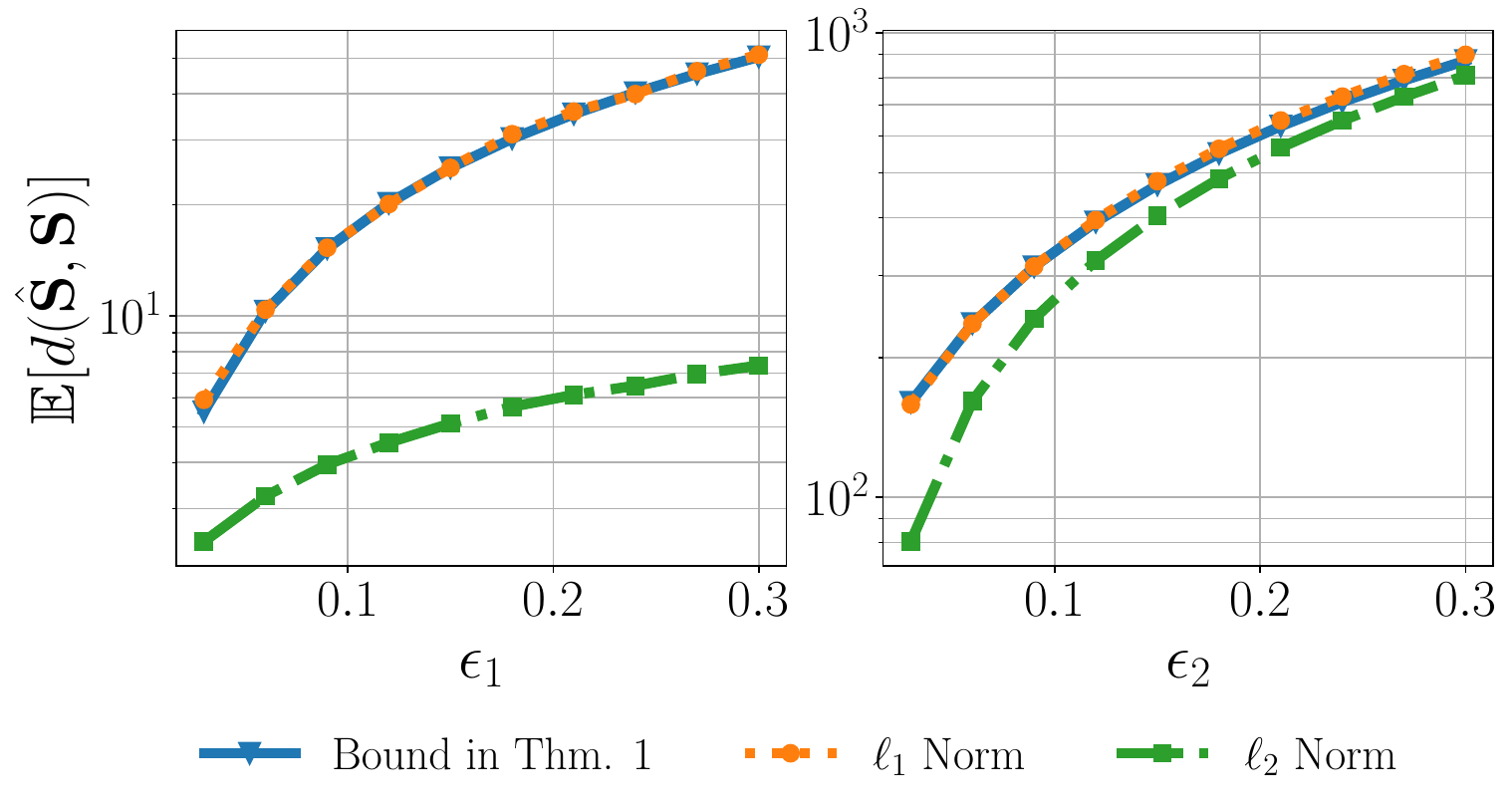}
    \caption{
    Theoretical (bound in Thm.~\ref{thm:Case1Adj}) and empirical bounds ($\ell_1$ and $\ell_2$ norms) for the perturbed Cora graph with $\bbS = \bbA$. 
    Left panel: varying $\epsilon_1$ with fixed $\epsilon_2=0$.
    Right panel: varying $\epsilon_2$ with fixed $\epsilon_1=0.5$.}
    \label{fig:theo_bdcrr_case1}
\end{figure}

\begin{figure}[!t]
    \centering 
    \includegraphics[width = 0.98\linewidth]{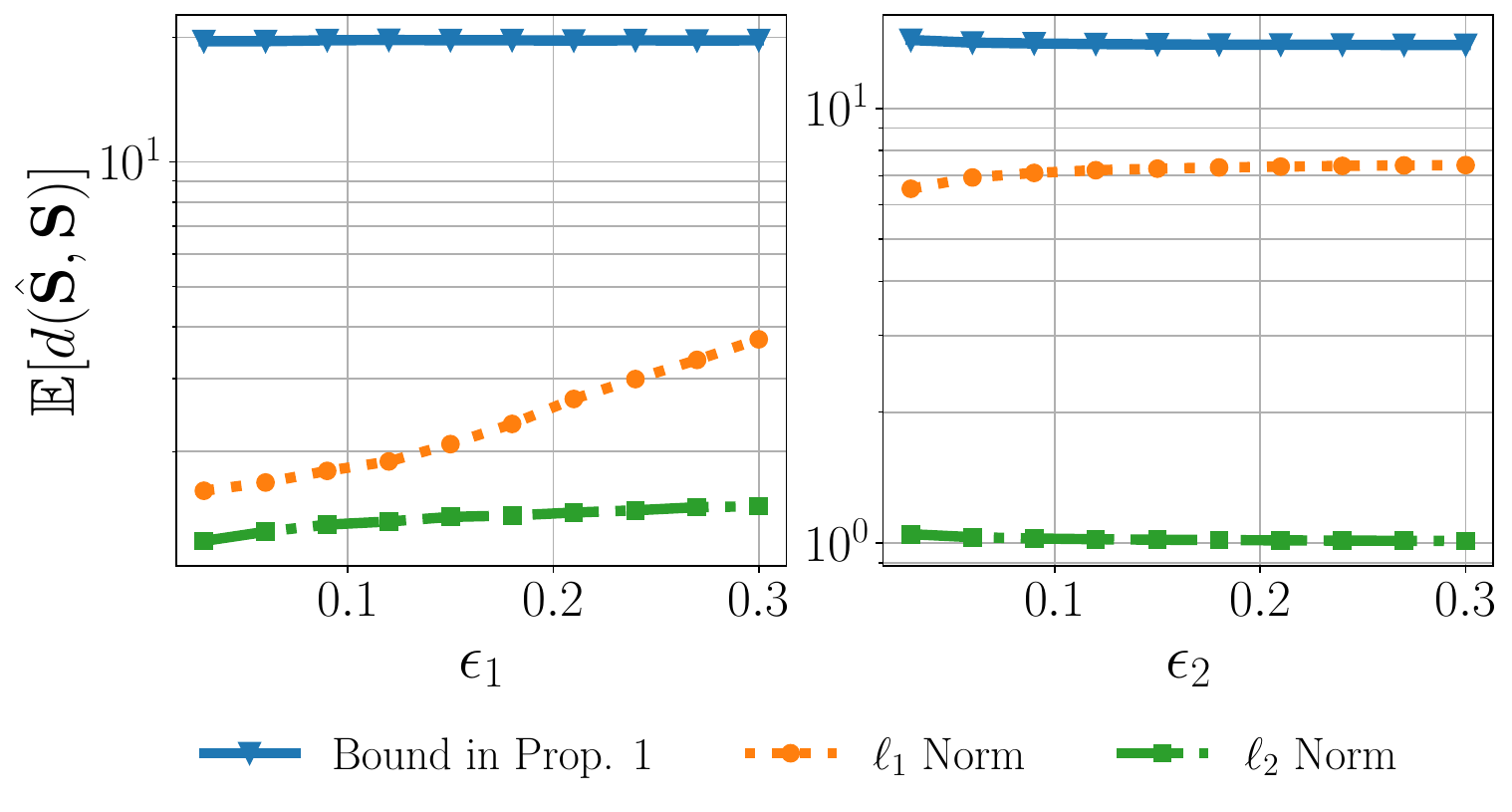}
    \caption{Theoretical (bound in Prop.~\ref{prop:Case2NorAdj}) and empirical bounds ($\ell_1$ and $\ell_2$ norms) for the perturbed Cora graph with $\bbS = \bbA_\textrm{n}$, under identical $(\epsilon_1, \epsilon_2)$ settings as Fig.~\ref{fig:theo_bdcrr_case1}.}
    \label{fig:theo_bdcrr_case2}
\end{figure}

\begin{figure*}[!t]
    \centering
    \includegraphics[width = 0.90\linewidth]{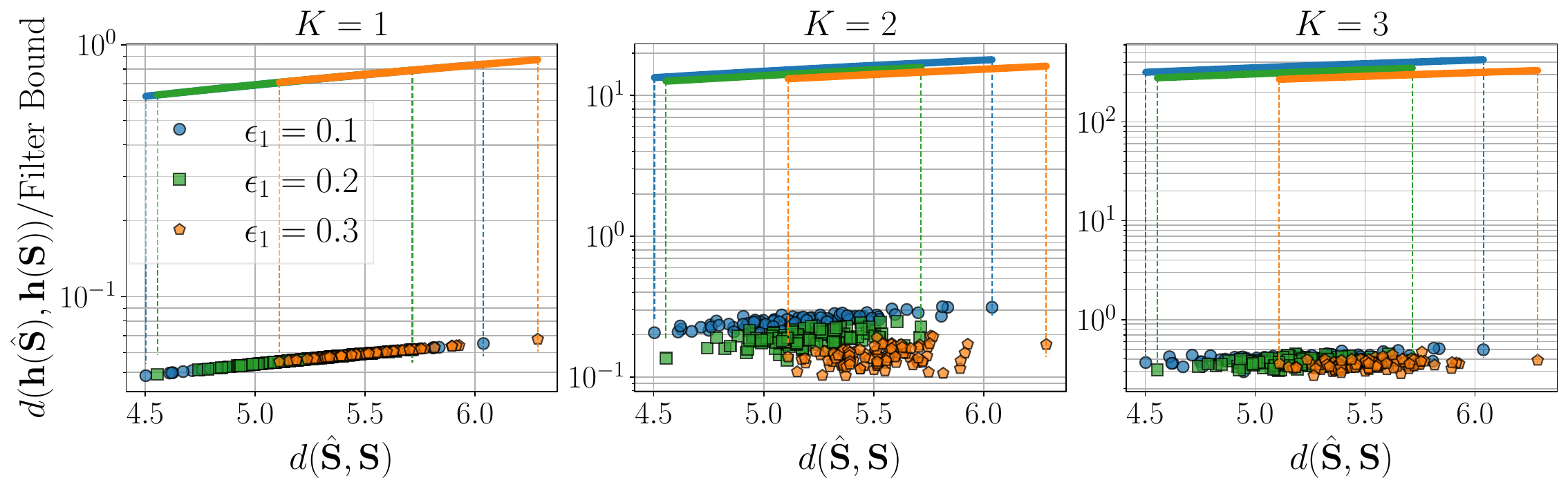}
    \caption{
    Comparison of Theorem~\ref{thm:gfdistance} bounds (solid lines) and empirical GF distances (scatter points)  with fixed $\epsilon_2=0.05$ and varying $\epsilon_1$ in $[0.1,0.2,0.3]$.}
    \label{fig:gf_sensitivity}
\end{figure*}

\begin{figure}[!t]
    \centering 
    \includegraphics[width = 0.98\linewidth]{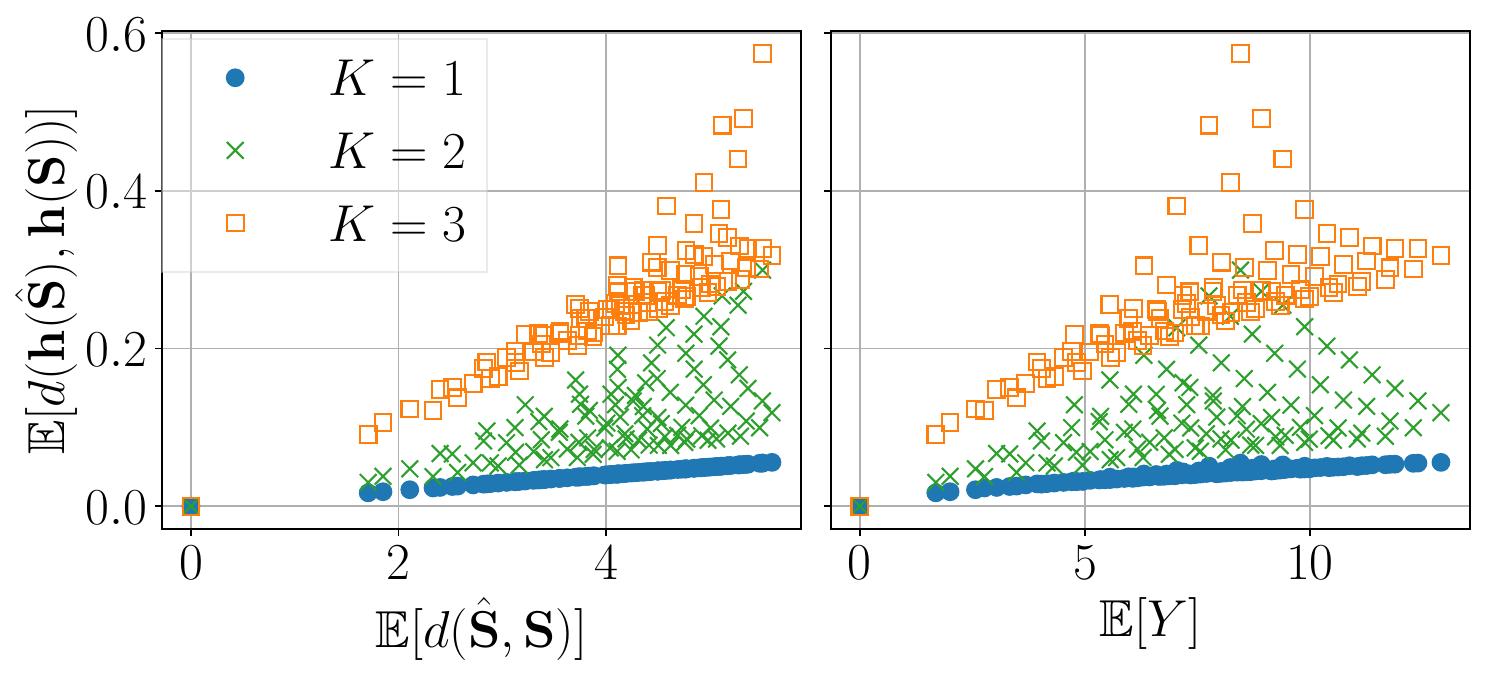}
    \caption{
    Expected GF output differences under increasing GF order and perturbation budget, illustrating intensified sensitivity along increased GF order and the alignment of Theorem \ref{thm:Case1Adj} bound with empirical GSO distance trends.
    }
    \label{fig:gf_sensitivity2}
\end{figure}

\subsection{Theoretical GSO Bound Corroboration}
\label{sub,sec:theoboundcoo}

\subsubsection{Synthetic graph}
We consider a two-group planted partition model (PPM), which is a special case of the stochastic block model.
Parameters are set with in-group probability to $p_{\rm in} = 0.8$, and between-group probability to $p_{\rm bet} = 0.5$.
The GSO is set as the unnormalized adjacency matrix $\bbS = \bbA$.
We perturb the PPM graph using the probabilistic error model~\eqref{eq:basic_ermodel} with two scales of perturbation budgets:
\begin{itemize}
    \item Small-scale perturbation (see Fig.~\ref{fig:theo_bdcrr_case0}, left panel): With $\epsilon_1 = 0.1$ and $\epsilon_2 = 0.01$, the graph is slightly altered, preserving its fundamental structure.
    \item Large-scale perturbation (see Fig.~\ref{fig:theo_bdcrr_case0}, right panel): With $\epsilon_1 = 0.5$ and $\epsilon_2 = 0.1$, the graph is under significant structural changes.
\end{itemize}
We carry out 101 Monte Carlo trials for varying graph sizes (ranging from $50$ to $1000$, in $50$-node increments). 
These simulations evaluate the expected bound from Theorem~\ref{thm:Case1Adj} and the deterministic bound from~\cite[Theorem 2]{Wang22-Eusipco} in relation to graph size. 
Comparisons with empirical GSO distances~\eqref{eq:gso_distance}, calculated using the $\ell_2$ norm, reveal that our expectation bound is consistently tighter than the deterministic counterpart from~\cite{Wang22-Eusipco}. 
This difference arises due to the consideration of degree changes and the probabilistic nature of our bound, as opposed to the worst-case scenario focus of the deterministic bound. 
Another observation is the increased bound magnitude correlating with higher perturbation budgets, as depicted in Fig.~\ref{fig:theo_bdcrr_case0}. 
Both bounds remain valid, even in high perturbation scenarios, underscoring the robustness of our theoretical frameworks.

\subsubsection{Real-life graph}

We utilize the undirected Cora citation graph \cite{Sen_Namata_Bilgic_Getoor_Galligher_Eliassi-Rad_2008}, which comprises $N = 2708$ nodes, and $C = 7$ classes.
Assuming the undirected nature of the underlying graph, we modify the original Cora graph from a directed to an undirected one.
The undirected Cora graph has $|\ccalE| = 5278$ edges.
We ascertain the evolution of our theoretical bounds against an increase in edge deletion probability $\epsilon_1$ and edge addition probability $\epsilon_2$. 
These alterations are systematically tracked along with using the $\ell_1$ and $\ell_2$ norms of the discrepancy between the original and perturbed graphs.

The range of $\epsilon_1$ and $\epsilon_2$ is set within $[3\times10^{-2}, 3\times10^{-1}]$, increasing in steps of $3\times10^{-2}$.
In each step, we compute the $\ell_1$ and $\ell_2$ norms of the difference between the original and perturbed adjacency matrices. 
We then compare these empirical results with the theoretical bounds provided in Theorem~\ref{thm:Case1Adj} and Proposition~\ref{prop:Case2NorAdj}.
In Fig.~\ref{fig:theo_bdcrr_case1}, with the GSO as the unnormalized adjacency matrix $\bbS = \bbA$, two distinct scenarios are presented: varying $\epsilon_1$  with $\epsilon_2=0$ (left panel), and varying $\epsilon_2$ with $\epsilon_1=0.5$ (right panel).
Through 101 Monte Carlo trials, the theoretical bound closely aligns with the empirical $\ell_1$ norm, particularly in scenarios where increased $\epsilon_2$ leads to denser graphs.
This trend suggests that enhanced precision of the bounds as graph densities shift from sparse to dense.

In Fig.~\ref{fig:theo_bdcrr_case2}, employing the normalized adjacency matrix~$\bbS = \bbA_\textrm{n}$ as the GSO, a similar analysis is conducted. 
In the left panel, an increase in $\ell_1$ and $\ell_2$ norm bounds is observed under rising error, and Proposition~\ref{prop:Case2NorAdj} gives a stable upper bound.
However, the accuracy of the bound is comparatively less satisfactory in the normalized case.
The right-hand case illustrates a stable empirical $\ell_2$ norm with an increasing number of edges, while the $\ell_1$ norm and our bound present slight increases and decreases, respectively. 
These observations can be attributed to the following factors:
\textit{(i)} the normalization operation keeps the adjacency matrix operator norm around~1; \textit{(ii)} an increased number of edges raises the $\ell_1$ norm; \textit{(iii)} increases in the denominator in Lemma~\ref{lmm:EuNormalized} result in a general decrease in the bound.

\subsection{GF Sensitivity Test}
\label{sub,sec:gf_sensitivity_test}
In this experiment, we evaluate the sensitivity of GF to the probabilistic error model.
We employ an ER graph with $N=100$ nodes and a connection probability of $0.1$ as the baseline graph.
The GSO is set as the unnormalized adjacency matrix~$\bbS = \bbA$.
Our focus is on the relationship between filter distance and the bound in Theorem~\ref{thm:gfdistance} for low pass GFs of orders $K=1,2,3$.
The findings are presented in Figs.~\ref{fig:gf_sensitivity} and~\ref{fig:gf_sensitivity2}.

In Fig.~\ref{fig:gf_sensitivity}, the edge addition probability is fixed as $\epsilon_2=0.05$ and the edge deletion probability $\epsilon_1$ varies among $[0.1,0.2,0.3]$.
Over $101$ Monte Carlo trials, we plot the empirical GF distances $d(\bbh(\hbS, \bbS))$ alongside the \revised{corresponding GSO distances $d(\hbS, \bbS) = \| \bbE \|$} as scatter plots.
These empirical GF distances demonstrate the linear scaling with the bounds in Theorem~\ref{thm:gfdistance}, depicted as solid lines.
It is noted that the tightness of these bounds decreases with an increase in the GF order.
The primary aim of this analysis is to confirm the linear relationship in Theorem~\ref{thm:gfdistance}.

In Fig.~\ref{fig:gf_sensitivity2}, the expected output differences of GFs $\mbE[d(\bbh(\hbS), \bbh(\bbS))]$ with orders $K=1,2,3$ are plotted against the expected GSO differences $\mbE[d(\hbS,\bbS)]$ and the bound in Theorem~\ref{thm:Case1Adj}.
Over $101$ Monte Carlo trials with perturbation probabilities $\epsilon_1 \in [0,0.3]$ and $ \epsilon_2 \in [0,0.05]$, the left panel shows that output differences increase with the GF order.
The right panel confirms that the bound  $\mbE[Y]$ captures trends similar to the empirical expectation of GSO distance, corroborating Theorem~\ref{thm:Case1Adj}. 
This suggests that for small, sparsely connected graphs, the sensitivity of a low pass GF to perturbations intensifies as its order increases.

\subsection{GCNN Sensitivity Test}
\label{sec:exp_gcnn}
%
\begin{figure}[t]
    \centering
    \includegraphics[width = 0.9\linewidth]{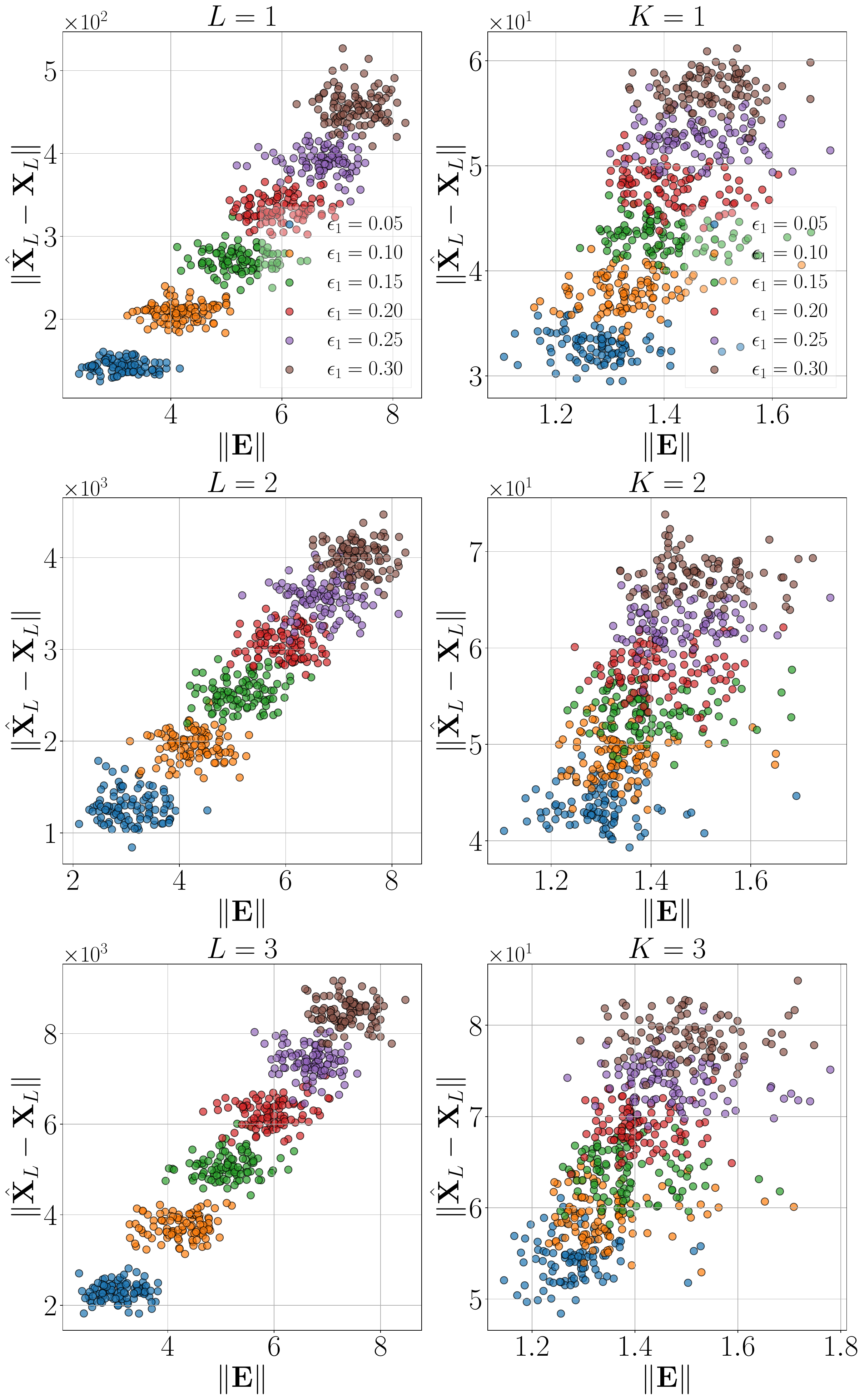}
    \caption{
    Correlation between GIN (left panel) and SGCN (right panel) output differences and GSO distances. 
    Analysis is conducted with varying edge deletion probabilities $\epsilon_1$, and a fixed edge addition probability $\epsilon_2 = 1\times10^{-4}$.
    }
    \label{fig:GINSGCN_sensitivity}
\end{figure}

%
\subsubsection{Linearity corroboration}
\label{sub,sub:GCNN_Linearity_Corroboration}
The experimental validation of Theorem~\ref{thm:gcnsensitivity} is conducted using GIN (Corollary~\ref{corr:GIN_MLP}) and SGCN (Corollary~\ref{corr:SGCN_sensitivity}).
We note that Corollary~\ref{corr:GIN_MLP} is only applicable for the single-layer GIN ($L=1$).
For the multi-layer GIN, our experiments show the recursion of linearity indicated in Theorem~\ref{thm:gcnsensitivity} empirically (see left panel of Fig.~\ref{fig:GINSGCN_sensitivity}).
These experiments are carried out on the Cora citation dataset, as discussed in Section~\ref{sub,sec:theoboundcoo}, to assess the sensitivity of GIN and SGCN to perturbed GSOs under evasion attacks. 

In Fig.~\ref{fig:GINSGCN_sensitivity}, for GIN (left panel), each layer comprises $16$ hidden features. 
GIN variants with $1$, $2$, and $3$ layers differ only in the number of cascaded graph filters with MLPs.
We investigate the correlation between empirical GIN output differences and GSO distances.
The edge deletion probability, $\epsilon_1$, is varied within $[5\times10^{-2}, 3\times10^{-1}]$ in increments of $5\times10^{-2}$, while the edge addition probability is fixed as $\epsilon_2 = 1 \times 10^{-4}$.
The results, categorized by edge deletion probability $\epsilon_1$, are obtained from 101 Monte Carlo trials, computing pairs of bounds and GIN output differences.
For SGCN (right panel), we examine networks of orders $K = [1,2,3]$ using a similar approach.
Empirical observations for $L=1,2,3$ and $K=1,2,3$ in GIN and SGCN demonstrate a linear correlation between output differences and GSO distances, corroborating the theoretical frameworks in Corollary~\ref{corr:GIN_MLP} and Corollary~\ref{corr:SGCN_sensitivity}.

Notably, the output differences observed in the two cases operate on different scales. 
For the SGCN with normalized GSO (right panel), the variation in output differences with increasing perturbation probability is more gradual compared to the unnormalized GSO used in GIN (left panel), which shows a steeper change. 
This discrepancy is likely due to the influence of the estimated GSO spectral norm $\lambda$.

\subsubsection{Accuracy drop under perturbation}
\label{sub,sub:GCNN_Accuracy_Corroboration}
\begin{figure*}[t]
    \centering
    \subfloat{
        \includegraphics[width = 0.5\linewidth]{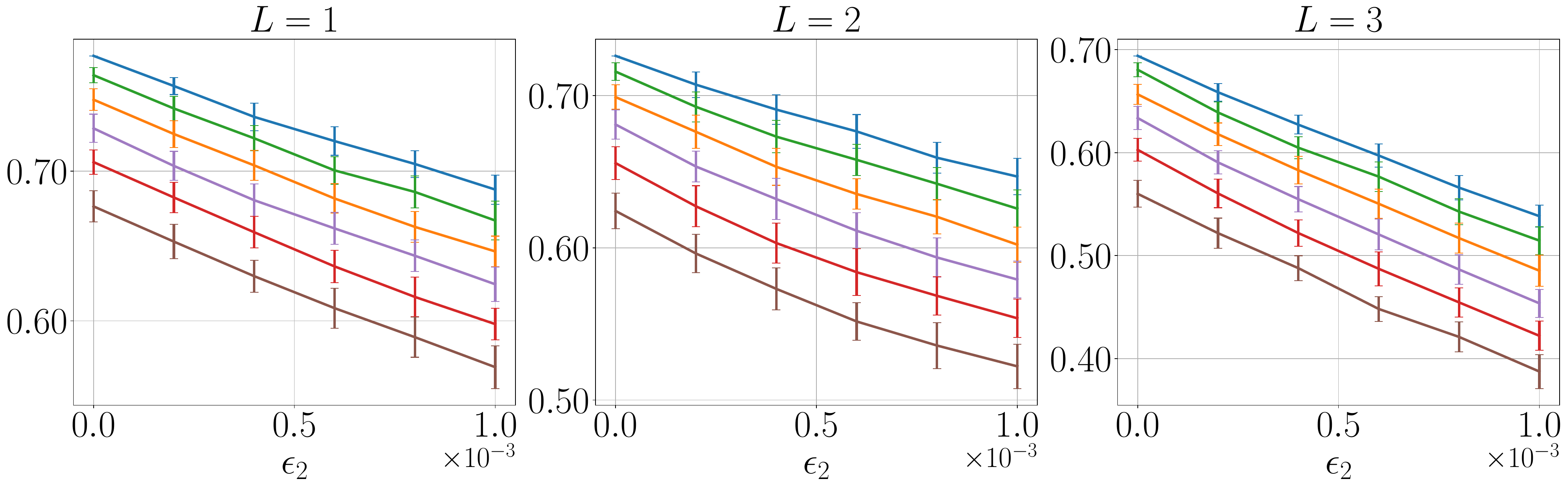}}
    \subfloat{
        \includegraphics[width = 0.5\linewidth]{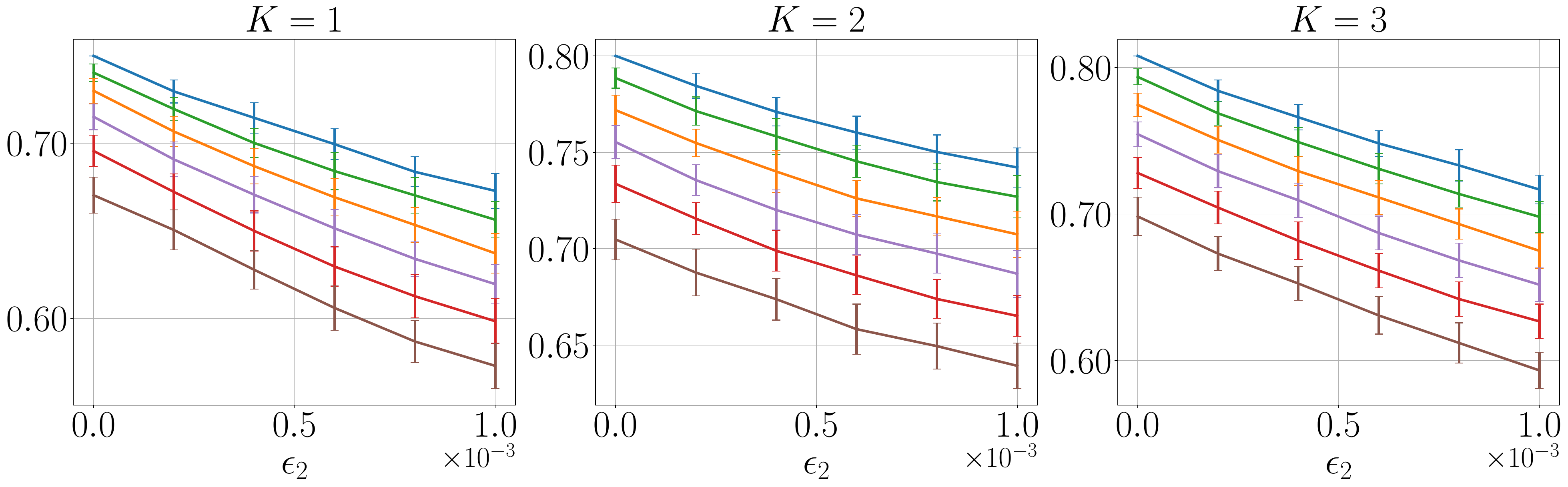}}
    \vspace{-3mm}
    \\
    \subfloat{
        \includegraphics[width = 0.5\linewidth]{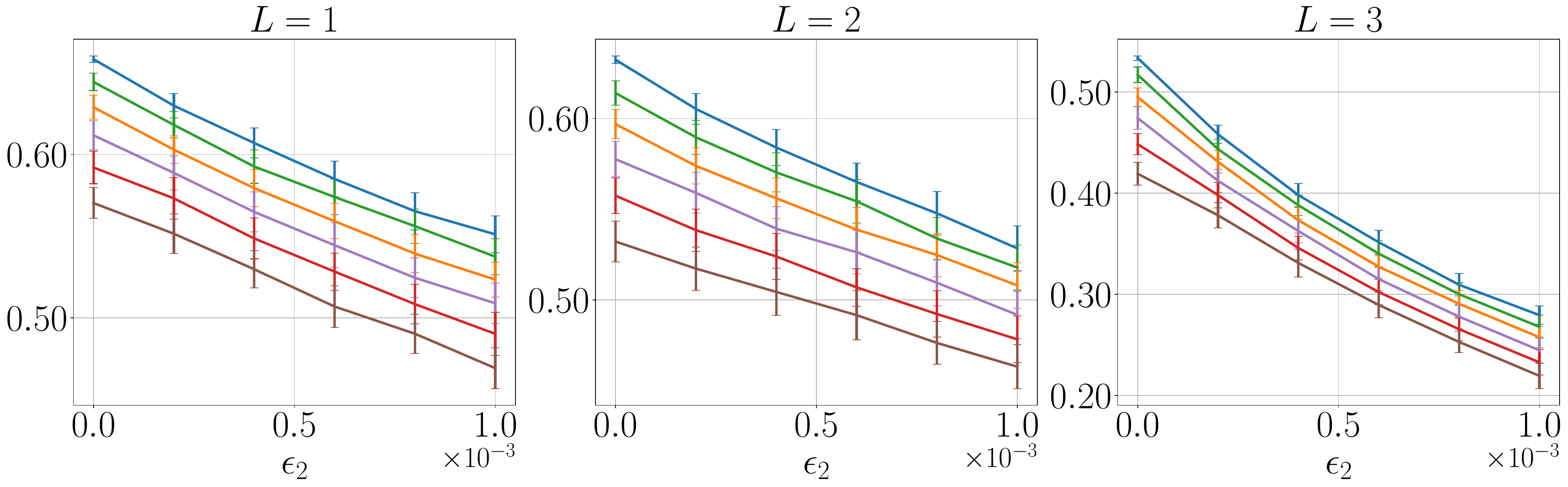}}
    \subfloat{
        \includegraphics[width = 0.5\linewidth]{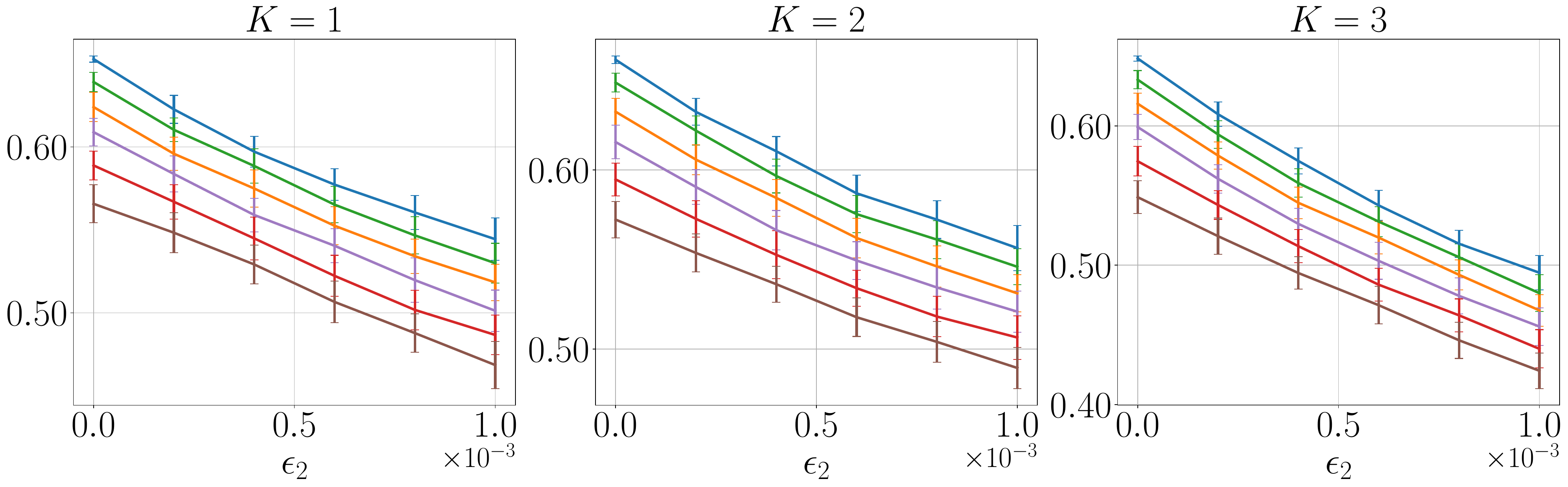}} 
    \vspace{-3mm}
    \\
    \setcounter{subfigure}{0} 
    \subfloat[GIN]{
    \label{fig:accuracy_drop_a}
        \includegraphics[width = 0.5\linewidth]{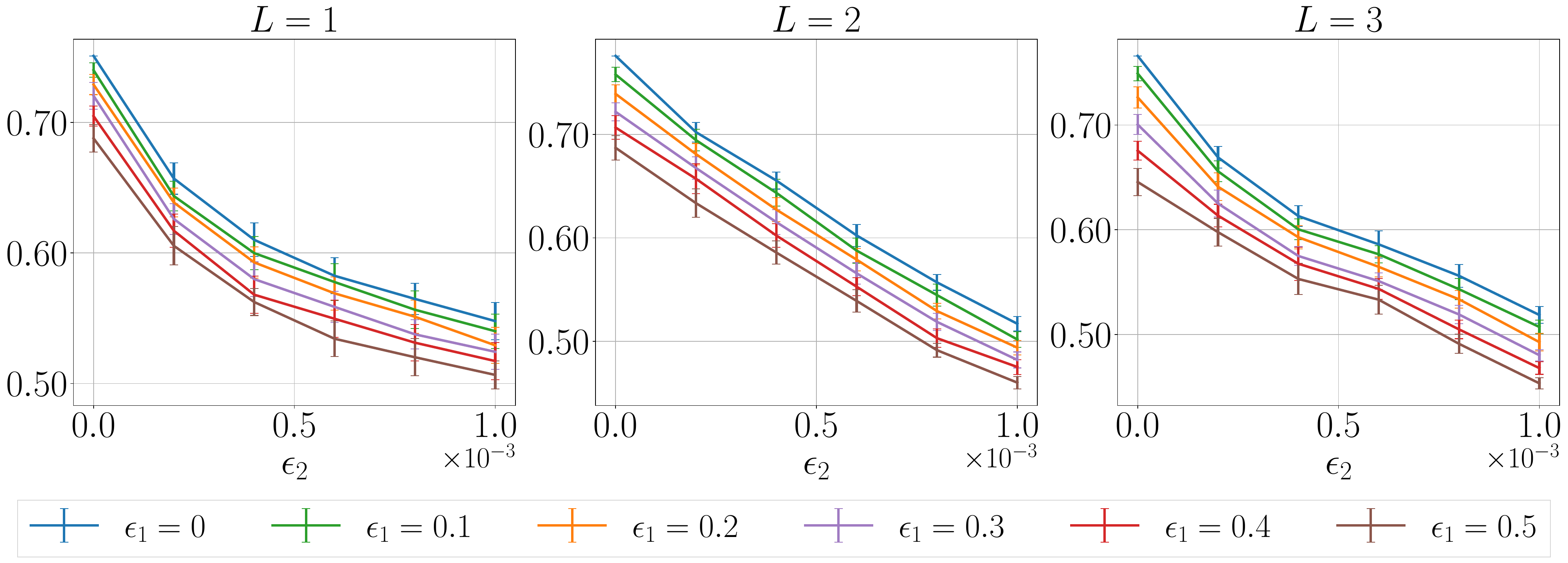}}
    \subfloat[SGCN]{
    \label{fig:accuracy_drop_b}
        \includegraphics[width = 0.5\linewidth]{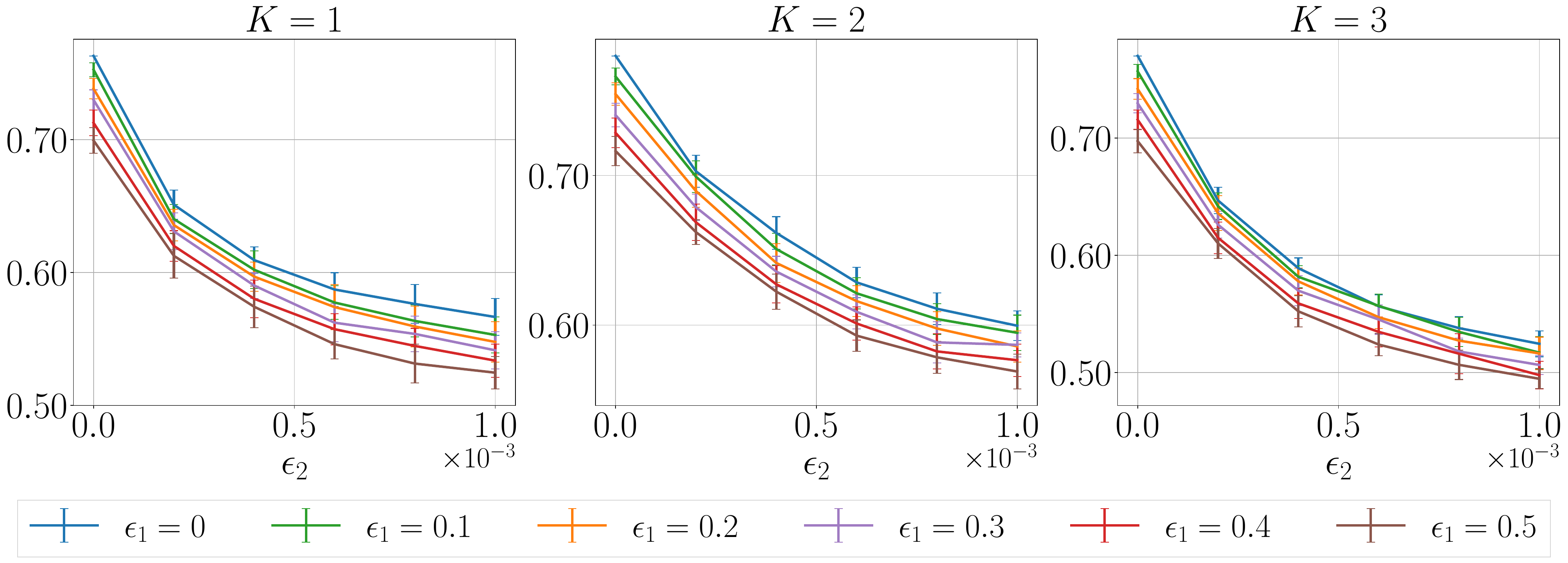}} 
    \caption{Accuracy changes for GIN ($L=1,2,3$) and SGCN ($K=1,2,3$) under perturbations, with $\epsilon_1 \in [0,0.5]$ and $\epsilon_2 \in [0,1\times10^{-3}]$.
    The 1st, 2nd and 3rd rows correspond to Cora, CiteSeer, and PubMed datasets, respectively.}
    \label{fig:accuracy_drop}
\end{figure*}

After affirming the linear sensitivity in Theorem~\ref{thm:gcnsensitivity}, we also examine the stability of GCNN under significant graph perturbations by observing the accuracy changes of same GCNN candidates as in Section~\ref{sub,sub:GCNN_Linearity_Corroboration}.

These experiments are conducted on three citation datasets: Cora, CiteSeer and PubMed~\cite{Sen_Namata_Bilgic_Getoor_Galligher_Eliassi-Rad_2008}.
The objective is to assess the impact of different perturbation budgets on the accuracy of GIN and SGCN models.
The perturbation budget parameters are set as follows: edge deletion probability $\epsilon_1$ varies within $[0,0.5]$ in increments of $0.1$, and edge addition probability $\epsilon_2$ varies within $[0,1\times10^{-3}]$ in increments of $2\times10^{-4}$.
Consistent with the experimental settings in Section~\ref{sub,sub:GCNN_Linearity_Corroboration}, the same GCNN candidates are utilized. 
The averaged accuracy results are shown in Fig.~\ref{fig:accuracy_drop}, where the bar indicates the standard variance of accuracy results.
The first, second and third rows correspond to datasets Cora, CiteSeer and PubMed, respectively.

A consistent pattern of accuracy decrease across all datasets and GCNN models is observed in Fig.~\ref{fig:accuracy_drop}, where the accuracy gradually decreases with increasing perturbation budgets. 
Notably, larger graphs (e.g., PubMed) exhibit a faster accuracy drop compared to smaller graphs (e.g., Cora and CiteSeer). 
This can be attributed to the alteration of more edges under the same perturbation budget in larger graphs.
When fixing edge deletion probability $\epsilon_1$, accuracy drops by approximately $10\%$ (as in Fig.~\ref{fig:accuracy_drop_a}, 1st row with $L=1$), and up to $20\%$ (as in Fig.~\ref{fig:accuracy_drop_a}, 3rd row with $L=3$).
With a fixed edge addition probability $\epsilon_2$, the accuracy drop is around $10\%$ (as in Fig.~\ref{fig:accuracy_drop_a}, 1st row with $L=1$), and approximately $5\%$ (as in Fig.~\ref{fig:accuracy_drop_a}, 3rd row with $L=1$).
This is likely because that, for sparse graphs, the same edge addition probability results in the addition of more edges than the number influenced by the same edge deletion probability.

The maximum of edge perturbation budget $\epsilon_1$ and $\epsilon_2$ is set to $0.5$ and $1\times10^{-3}$, respectively.
Consequently, up to $50\%$ of the edges are deleted, and $70\%$ are added relative to the original edge count.
In this case, the graph structure is significantly perturbed.
This significant graph perturbation makes the accuracy drop by up to $20\%$.
Under such large perturbations, GCNN gives finite responses. 
\revised{Thus, the GCNN is stable in our context even when the downstream task performance is significantly impacted, which is due to large-scale edge perturbations.}
This also verifies Theorem~\ref{thm:gcnsensitivity}, where it is stated that as long as the GSO perturbation is bounded/finite, the GCNN output difference is also bounded/finite.

\section{Conclusion and Discussion}
\label{sec:Conclusion}
This paper has presented an analytical framework for investigating the sensitivity of GCNNs to GSO perturbations, employing a probabilistic graph perturbation model.
We have established tighter error bounds than those previously available.
We have theoretically demonstrated that the expected output variation for a single layer of GCNN is linearly bounded by the GSO error, ensuring the stability (bounded output difference) of single-layer GCNN under bounded GSO errors.
For multilayer GCNN, our analysis has shown that the dependency of GCNN output difference on GSO error can be described through a recursion of linearity.
Specifically, this dependency is explicitly controlled by: the input feature, the GSO, error model parameters, Lipschitz constants of activation functions in GCNN, and GCNN weights.
Through numerical experiments, we have validated our theoretical findings and confirmed that GCNNs (exemplified with GIN and SGCN) maintain stability under large-scale graph edge perturbations, despite significant performance reductions.

In this work, our primary focus is on edge perturbations in graphs, while potential modifications to the graph signal and node injections are not considered.
Any alterations to the graph signal could be subsumed within the spectral norm when performing sensitivity analysis. 
However, node injection presents a challenge that cannot be addressed using the current definition of graph distance. 
This is due to the discrepancy in sizes between the unperturbed and perturbed graphs as the number of nodes increases.
A potential solution to this issue could involve redefining the GSO distance using a different metric. 
In this context, Optimal Transport (OT) and its variants emerge as viable candidates for this task~\cite{Lenaic18unOT, chapel20partialOT, maretic2022wasserstein}. 
These methods allow for the augmentation of a smaller graph, facilitating the establishment of a meaningful graph distance metric~\cite{Chuang22TMD}.
Consequently, future research could explore an encompassing approach that considers all of the aforementioned types of graph perturbations. 
Such an investigation has the potential to yield more comprehensive insights into the stability of GCNNs under perturbations.

Graph regularization methods are commonly used to achieve robust graph learning and estimation~\cite{sun22atksurvey}.
Research on adversarial training of GCNNs typically uses specifically designed loss functions to strengthen GCNNs against structural and feature perturbations, thus improving their performance stability against certain graph disturbances~\cite{jin2019latent, feng2019graph, dai2019adversarial, ren2021integrated, zhao2021expressive}.
In graph learning, several techniques have been developed to regulate graphs and signals based on specific graph signal assumptions to perform graph estimation~\cite{Dong16-Learning,segarra2017network,egilmez2018graph,pu2021kernel}.
With the inclusion of effective graph regularization, our sensitivity analysis offers insight that can contribute to the development of a uniform metric, paving the way for a more transferable and robust GCNN.

\appendices
\section{Upper Bound of $\|\bbE_u\|_1$} \label{lmm:EuNormalized_proof}
\begin{proof}[Proof of Lemma \ref{lmm:EuNormalized}]
We start with the first term in \eqref{eq:KenlayGeneralBound}, which is bounded by $\tau_u \leq d_v$
\begin{equation}\label{eq:KenlayBoundDelete}
    \sum_{v \in \ccalD_u}\dfrac{1}{\sqrt{d_ud_v}} \leq \sum_{v \in \ccalD_u}\dfrac{1}{\sqrt{d_u\tau_u}} = \dfrac{\delta_u^-}{\sqrt{d_u\tau_u}}.
\end{equation}
The second and third terms in \eqref{eq:KenlayGeneralBound} can be bounded using triangle inequality as follows
\begin{align} \label{eq:normalized_bound_process1}
    & \sum_{v \in \ccalA_u}\dfrac{1}{\sqrt{\hhatd_u\hhatd_v}}
    + \sum_{v \in \ccalR_u}\left|\dfrac{1}{\sqrt{d_ud_v}} - \dfrac{1}{\sqrt{\hhatd_u\hhatd_v}}\right| \nonumber \\
    & \leq \sum_{v \in \ccalA_u}\dfrac{1}{\sqrt{\hhatd_u\hhatd_v}}
    + \sum_{v \in \ccalR_u} \left( \dfrac{1}{\sqrt{d_ud_v}} + \dfrac{1}{\sqrt{\hhatd_u\hhatd_v}} \right) \nonumber \\
    & = \sum_{v \in \ccalR_u} \dfrac{1}{\sqrt{d_ud_v}} + 
    \sum_{v \in \ccalA_u\cup\ccalR_u}\dfrac{1}{\sqrt{\hhatd_u\hhatd_v}}.
\end{align}
For the first term in \eqref{eq:normalized_bound_process1}, we have
\begin{align}
    \sum_{v \in \ccalR_u} \dfrac{1}{\sqrt{d_ud_v}} 
    \leq \sum_{v \in \ccalR_u}  \dfrac{1}{\sqrt{d_u\tau_u}} 
    \leq \dfrac{d_u - \delta_u^-}{\sqrt{d_u\tau_u}}.
\end{align}
For the second term in \eqref{eq:normalized_bound_process1}, we have
\begin{align}
    & \sum_{v \in \ccalA_u\cup\ccalR_u}\dfrac{1}{\sqrt{\hhatd_u\hhatd_v}} \nonumber \\
    & = \sum_{v \in \ccalA_u\cup\ccalR_u}\dfrac{1}{\sqrt{(d_u+\delta_u^+ -\delta_u^-)(d_v + \delta_v^+ -\delta_v^-)}} 
\end{align}
Thus, we have a new bound, which is more suited to our error model, that is
\begin{equation}\label{eq:PreNewBound1}
    \begin{split}
        & \|\bbE_u\|_1 \leq \dfrac{\delta_u^-}{\sqrt{d_u\tau_u}} + \dfrac{d_u - \delta_u^-}{\sqrt{d_u\tau_u}}\\
        & + \sum_{v \in \ccalA_u\cup\ccalR_u}\dfrac{1}{\sqrt{(d_u+\delta_u^+ -\delta_u^-)(d_v + \delta_v^+ -\delta_v^-)}} \\
        & = \sqrt{d_u/\tau_u}
        + \sum_{v \in \ccalA_u\cup\ccalR_u}\dfrac{1}{\sqrt{(d_u+\delta_u^+ -\delta_u^-)(d_v + \delta_v^+ -\delta_v^-)}}.
    \end{split}
\end{equation}
We will adapt the general bound \eqref{eq:PreNewBound1} to the probabilistic error model presented in \eqref{eq:basic_model}.
In \eqref{eq:PreNewBound1}, we let
\begin{equation}
    \begin{split}
        & Z_{u,1}=\sqrt{d_u/\tau_u}, \\
    & Z_{u,2}=\sum_{v \in \ccalA_u\cup\ccalR_u}\frac{1}{\sqrt{(d_u+\delta_u^+ -\delta_u^-)(d_v + \delta_v^+ -\delta_v^-)}},
    \end{split}
\end{equation}
where $\delta_u^- \sim \textrm{Bin}(d_u, \epsilon_1)$, $\delta_u^+ \sim \textrm{Bin}(d_u^*, \epsilon_2)$, $\delta_v^- \sim \textrm{Bin}(d_v, \epsilon_1)$, $\delta_v^+ \sim \textrm{Bin}(d_v^*, \epsilon_2)$, $d_u^* = N-d_u-1$ and $d_v^* = N-d_v-1$.
Finally, we obtain 
\begin{align}\label{eq:f_uDeltas}
    \|\bbE_u\|_1 \leq Z_{u,1}+ Z_{u,2}.
\end{align}
This completes the proof.
\end{proof}
%

\section{Graph filter sensitivity} \label{apdx:gfdistance_proof}
\begin{proof}[Proof of Theorem \ref{thm:gfdistance}]
First, we recall the following result.
\begin{lemma}
    \label{lemma:powernormbound}
    {\rm (Lemma 3, \cite{Levie19-Transferability})} Suppose that $\hat{\bbS}, \bbS, \bbE \in \mathbb{R}^{N \times N}$ are Hermitian matrices satisfying $\hat{\bbS} = \bbS + \bbE$, and $\lambda = \max \{\|\hbS \|, \|\bbS \| \}$. 
    Then for every $k \geq 0$
    \begin{equation}
        \|\hat{\bbS}^k - \bbS^k \| = \|(\bbS + \bbE)^k - \bbS^k \|  \leq k \lambda^{k-1} \|\bbE \|.
    \end{equation}
\end{lemma}
\noindent Expand the filter representation in $\| \bbh(\hbS) - \bbh(\bbS) \|$, as
\begin{equation}
    \label{eqproof:fd-1}
    \begin{split}
        & \left\| \bbh(\hbS) - \bbh(\bbS) \right\|
        = \left\| \sum_{k=0}^{K} \left( h_k\hbS^k - h_k\bbS^k \right) \right\|.
    \end{split}
\end{equation}
By Lemma \ref{lemma:powernormbound} and repeatably using triangle inequality, \eqref{eqproof:fd-1} is bounded by
\begin{equation}
    \label{eqproof:filter_deterministic}
    \begin{split}
        & \left\| \sum_{k=0}^{K} \left( h_k\hbS^k - h_k\bbS^k \right) \right\| 
        \leq \sum_{k=0}^{K} |h_k| \| \hbS^k - \bbS^k \| \\
        & \leq  \sum_{k=0}^{K} |h_k| k \lambda^{k-1} \| \bbE \| 
        = \sum_{k=1}^{K} |h_k| k \lambda^{k-1} \| \bbE \|.
    \end{split}
\end{equation}
The correlation between $\lambda$ and $\|\bbE\|$ has two cases:
\begin{enumerate}
    \item If $\lambda= \|\bbS\|$,
    \begin{equation}\label{eq:case1_expeccov0}
        \mathbb{E}[\lambda^{k-1}\|\mathbf{E}\|] =  \mathbb{E}[\lambda^{k-1}] \mathbb{E}[\|\mathbf{E}\|];
    \end{equation}
    \item If $\lambda=\|\hbS\|$,
    \begin{equation}\label{eq:case2_expeccovnot0}
        \mathbb{E}[\lambda^{k-1}\|\mathbf{E}\|] =  \mathbb{E}[\lambda^{k-1}] \mathbb{E}[\|\mathbf{E}\|] + \textrm{Cov}[\|\mathbf{E}\|,\lambda^{k-1}].
    \end{equation}
\end{enumerate}
The following proof is based on the second case \eqref{eq:case2_expeccovnot0} because the covariance term can be set to zero to include the first case. 
By using \eqref{eqproof:fd-1} and taking the expectation of \eqref{eqproof:filter_deterministic}, we obtain
\begin{equation}\label{eqproof:filter_expetation_1}
    \begin{split}
        & \mbE \left[ \left\| \bbh(\hbS) - \bbh(\bbS) \right\| \right] 
        \leq  \mbE \left[ \sum_{k=1}^{K} |h_k| k \lambda^{k-1} \| \bbE \| \right]  \\
        & \leq  \sum_{k=1}^{K} k |h_k| \mbE \left[ \lambda^{k-1} \| \bbE \| \right] \\
        & = \sum_{k=1}^{K} k |h_k| 
        \left( \mathbb{E}[\lambda^{k-1}] \mathbb{E}[\|\mathbf{E}\|] + \textrm{Cov}[\|\mathbf{E}\|,\lambda^{k-1}] \right).
    \end{split}
\end{equation}
In \eqref{eqproof:filter_expetation_1}, let 
\begin{align}
    \lambda_{k} & = \mbE[\lambda^{k-1}], \\
    \zeta_k & = \textrm{Cov}[\|\bbE\|,\lambda^{k-1}]. \label{eq:covbound}
\end{align}
Then, we have
\begin{equation}\label{eqproof:filter_expetation_2}
    \begin{split}
        \mbE \left[ \left\| \bbh(\hbS) - \bbh(\bbS) \right\| \right] 
        \leq \sum_{k=1}^{K} k |h_k|
        \left( \lambda_{k} \mathbb{E}[\|\mathbf{E}\|] + \zeta_k \right).
    \end{split}
\end{equation}
This completes the proof.
\end{proof}

\section{GCNN Sensitivity} \label{apdx:gcnsensitivity_proof}
\begin{proof}[Proof of Theorem \ref{thm:gcnsensitivity}]
\textbf{First Layer.} At the first layer $\ell = 1$, the graph convolution is performed as follows
\begin{align}
\bbY_{1} = \sum_{k=1}^K \bbS^k\bbX_{0}\bbH_{1 k}, \quad \bbX_1 = \sigma_1(\bbY_{1}).
\end{align}
For a perturbed GSO $\hbS$, the difference between the perturbed and clean graph convolutions is
\begin{align}\label{eq:GenGCN_difflayer1-1}
\hbY_{1} - \bbY_{1} = \sum_{k=1}^K(\hbS^k-\bbS^k)\bbX_0\bbH_{1k}.
\end{align}
Using Lemma \ref{lemma:powernormbound}, we can bound \eqref{eq:GenGCN_difflayer1-1} as follows
\begin{align}\label{eq:GenGCN_difflayer1-2}
\left\| \hbY_{1} - \bbY_{1} \right\| \leq \sum_{k=1}^K k\lambda^{k-1}\|\bbX_0\| \|\bbH_{1k}\| \|\bbE\|.
\end{align}
Similar to giving the upper bound for the expectation of graph filter distance from \eqref{eqproof:filter_deterministic} to \eqref{eqproof:filter_expetation_2}, given the constants $\lambda_{k} = \mbE[\lambda^{k-1}]$ and $\zeta_k = \textrm{Cov}[\|\bbE\|,\lambda^{k-1}]$, we take the expectation of \eqref{eq:GenGCN_difflayer1-2} and obtain
\begin{align}
& \mbE\left[ \left\| \hbY_{1} - \bbY_{1} \right\| \right] \leq \mbE\left[ \sum_{k=1}^K k\lambda^{k-1}\|\bbX_0\| \|\bbH_{1k}\| \|\bbE\| \right] \notag \\
& = \sum_{k=1}^K k \|\bbX_0\| \|\bbH_{1k}\| \mbE\left[ \lambda^{k-1} \|\bbE\| \right] \notag \\
& = \sum_{k=1}^K k\|\bbX_0\| \|\bbH_{1k}\| 
\left( \mbE[\lambda^{k-1}] \mbE\left[ \|\bbE\| \right] +  \textrm{Cov}[\|\bbE\|,\lambda^{k-1}] \right) \notag \\
& \leq \sum_{k=1}^K k\|\bbX_0\| \|\bbH_{1k}\| \left(  \lambda_{k} \mbE\left[ \|\bbE\| \right] +  \zeta_k \right).
\label{eq:GenGCN_diff_exp_MClamda_fixed}
\end{align}
For simplicity, let $B_1 = \sum_{k=1}^K k \lambda_{k} \|\bbX_0\| \|\bbH_{1k}\|$, and $D_1 = \sum_{k=1}^K k\zeta_k \|\bbX_0\| \|\bbH_{1k}\|$.
Thus, \eqref{eq:GenGCN_diff_exp_MClamda_fixed} illustrates that the expectation of the graph filter distance at the first layer is bounded by a polynomial of $\mbE\left[\|\bbE\|\right]$ as 
\begin{align}
    \mbE\left[ \left\| \hbY_{1} - \bbY_{1} \right\| \right]  \leq B_1\mbE[\|\bbE\|] + D_1.
\end{align}
Consider the nonlinearity function $\sigma_1(\cdot)$ at the first layer, which satisfies the Lipschitz condition
\begin{align}
\|\sigma_1(\hbY) - \sigma_1(\bbY)\| \leq C_{\sigma_1} \| \hbY - \bbY \|.
\end{align}
Applying this Lipschitz condition to \eqref{eq:GenGCN_difflayer1-2}, we have
\begin{align}
& \mbE\left[ \| \hbX_1 - \bbX_1 \| \right]
= \mbE\left[ \left\| \sigma_1(\hbY) - \sigma_1(\bbY) \right\|\right] \notag \\
& \leq C_{\sigma_1} \mbE\left[ \left\| \hbY - \bbY \right\|\right]
\leq C_{\sigma_1} B_1 \mbE\left[\|\bbE\|\right] + C_{\sigma_1} D_1. 
\label{eq:GenGCN_difflayer2-1}
\end{align}
\textbf{Second Layer.} 
At the second layer $\ell=2$, the graph convolution is performed as
\begin{align}
\bbY_2 = \sum_{k=1}^K \bbS^k\bbX_{1}\bbH_{2k}, \quad \bbX_2 = \sigma(\bbY_2).
\end{align}
The difference between the perturbed and clean graph convolutions is given by
\begin{align}\label{eq:GenGCN_difflayer2}
& \hbY_2 - \bbY_2
= \sum_{k=1}^K\hbS^k\hbX_1\bbH_{2k} - \sum_{k=1}^K\bbS^k\bbX_1\bbH_{2k} \notag \\
& = \sum_{k=1}^K (\hbS^k\hbX_1 - \hbS^k\bbX_1 + \hbS^k\bbX_1 - \bbS^k\bbX_1) \bbH_{2k} \notag \\
& = \sum_{k=1}^K \left( \hbS^k(\hbX_1 - \bbX_1) + (\hbS^k - \bbS^k)\bbX_1 \right) \bbH_{2k}.
\end{align}
Taking the expectation of \eqref{eq:GenGCN_difflayer2} and using \eqref{eq:case2_expeccovnot0}, Lemma \ref{lemma:powernormbound} as well as the submultiplicativity of the spectral norm, we have
\begin{align}
& \mbE \left[ \left\| \hbY_2 - \bbY_2 \right\| \right] \notag \\
& \leq \mbE \left[ \left\| \sum_{k=1}^K \left( \hbS^k(\hbX_1 - \bbX_1) + (\hbS^k - \bbS^k)\bbX_1 \right) \bbH_{2k} \right\| \right] \notag \\
& \leq \sum_{k=1}^K \|\bbH_{2k}\| \mbE \left[ \left\| \hbS^k(\hbX_1 - \bbX_1) \right\| + \left\| (\hbS^k - \bbS^k)\bbX_1 \right\| \right] \notag \\
& \revised{\leq} \sum_{k=1}^K \|\bbH_{2k}\| \Big( \mbE[\lambda^k] \mbE \left[ \|\hbX_1 - \bbX_1 \| \right] + \Cov\left[\|\hbX_1 - \bbX_1\|, \lambda^k\right] \notag \\
& + k\|\bbX_1\| \left( \mbE[\lambda^{k-1}] \mbE \left[ \|\bbE\| \right] 
+ \textrm{Cov}[\|\bbE\|,\lambda^{k-1}] \right) \Bigr). 
\label{eq:GenGCN_difflayer2-2}
\end{align}
Let 
\begin{align}
    \mu_{k,\ell-1} = \Cov[\|\hbX_{\ell-1} - \bbX_{\ell-1}\|, \lambda^k],
\end{align}
where $k=1,\ldots,K$, and $\ell = 2,\ldots,L$.
Thus, in \eqref{eq:GenGCN_difflayer2-2}, we have $\mu_{k,1} = \Cov \|\hbX_1 - \bbX_1\|$.
Then,
we can express \eqref{eq:GenGCN_difflayer2-2} as a function controlled by $\mbE[\|\bbE\|]$
\begin{align}
    & \mbE \left[ \left\| \hbY_2 - \bbY_2 \right\| \right] 
    \leq \sum_{k=1}^K \|\bbH_{2k}\| \Bigl( \lambda_{k+1} \mbE \left[ \|\hbX_1 - \bbX_1 \| \right] \nonumber \\ 
    & + \mu_{k,1} + k\lambda_k\|\bbX_1\| \mbE [\|\bbE\|]  + k \zeta_k\|\bbX_1\| \Bigr) \nonumber  \\
    & \leq \sum_{k=1}^K \|\bbH_{2k}\| \Bigl( \left( \lambda_{k+1}  C_{\sigma_1} B_1 + k\lambda_k\|\bbX_1\| \right) \mbE [\|\bbE\|] \nonumber \\
    & + \mu_{k,1} + \lambda_k C_{\sigma_1} D_1 + k \zeta_k\|\bbX_1\| \Bigr) \nonumber \\
    & \leq B_2\mbE [\|\bbE\|]  + D_2,
\end{align}
where $B_2 = \sum_{k=1}^K \left( \lambda_{k+1}  C_{\sigma_1} B_1 + k\lambda_k\|\bbX_1\| \right) \|\bbH_{2k}\| $ and $D_2 = \sum_{k=1}^K \left( \mu_{k,1}+ \lambda_k C_{\sigma_1} D_1 + k \zeta_k\|\bbX_1\| \right) \|\bbH_{2k}\|$.
Consider the second layer's nonlinearity function $\sigma_2(\cdot)$, we have 
\begin{align}
    \mbE\left[ \|\hbX_2 - \bbX_2\| \right] \leq C_{\sigma_2}B_2\mbE [\|\bbE\|]  + C_{\sigma_2} D_2.
\end{align}

\noindent \textbf{Generalization to Layer $\ell \geq 1$.}
By induction, we can generalize the result to the output difference at any layer $\ell \geq 1$
\begin{align}
\mbE\left[ \left\| \hbX_{\ell} - \bbX_{\ell} \right\| \right] \leq C_{\sigma_\ell} B_\ell \mbE\left[ \| \bbE \| \right] + C_{\sigma_\ell} D_\ell,
\end{align}
where
\begin{equation}
    \begin{split}
        B_\ell & = \sum_{k=1}^K \left( \lambda_{k+1}  C_{\sigma_{\ell-1}} B_{\ell-1} + k\lambda_k\|\bbX_{\ell-1}\| \right) \|\bbH_{\ell k}\|, \\
        D_\ell & = \sum_{k=1}^K \left( \mu_{k,\ell-1}+ \lambda_k C_{\sigma_{\ell-1}} D_{\ell-1} + k \zeta_k\|\bbX_{\ell-1}\| \right) \|\bbH_{\ell k}\|.
    \end{split}
\end{equation}
This completes the proof.
\end{proof}

%
\section{Single-layer GIN Sensitivity} 
\label{apdx:ginpercepsensitivity_proof}
\begin{proof}
In a single-layer GIN, we assume that the inner MLP has two layers as earlier introduced in the paper.
The outputs of a single-layer GIN ($L=1$) with original and perturbed GSOs are given as
\begin{align} 
    \label{eq:proof_gin_general_ori}
    \bbX_{L} = \bbh_{\bbTheta_L} \bigl(\bbS \bbX_{L-1}\bigr), \\
    \label{eq:proof_gin_general_per}
    \hbX_{L} = \bbh_{\bbTheta_L} \bigl(\hbS \hbX_{L-1}\bigr).
\end{align}
Expanding \eqref{eq:proof_gin_general_ori} and \eqref{eq:proof_gin_general_per} with full matrix transformations, we have
\begin{align} \label{eq:proof_gin_full_ori}
\bbX_{L}
& = \sigma_{L 2} ( \sigma_{L 1} ( \bbS \bbX_{L-1} \bbW_{L1} + \bbB_{L1} )\bbW_{L 2} + \bbB_{L2} ), \\
\label{eq:proof_gin_full_per}
\hbX_{L}
& = \sigma_{L 2} ( \sigma_{L 1} ( \hbS \hbX_{L-1} \bbW_{L1} + \bbB_{L1} )\bbW_{L 2} + \bbB_{L2} ).
\end{align}
We can split \eqref{eq:proof_gin_full_ori} as
\begin{subequations}
    \begin{align}
        & \bbY_{L 1} = \bbS \bbX_{L-1} \bbW_{L1} + \bbB_{L1}, \label{eq_proof_gin_mlp_ori_a}\\
        & \bbX_{L 1} = \sigma_{L 1} ( \bbY_{L 1} ), \label{eq_proof_gin_mlp_ori_b}\\
        & \bbY_{L 2} = \bbX_{L1} \bbW_{L2} + \bbB_{L2}, \label{eq_proof_gin_mlp_ori_c}\\
        & \bbX_{L} = \sigma_{L 2} ( \bbY_{L 2} ), \label{eq_proof_gin_mlp_ori_d}
    \end{align}
\end{subequations}
where $\bbX_{L 1}$ denotes the intermediate output of the first layer, and $\bbX_{L}$ represents the output of the second layer. 
For simplicity of notation, we use $\bbX_{L}$ instead of $\bbX_{L2}$.
Similarly, we split \eqref{eq:proof_gin_full_per} as
\begin{subequations}
    \begin{align}
        & \hbY_{L 1} = \hbS \hbX_{L-1} \bbW_{L1} + \bbB_{L1}, \label{eq_proof_gin_mlp_per_a} \\
        & \hbX_{L 1} = \sigma_{L 1} ( \hbY_{L 1} ), \label{eq_proof_gin_mlp_per_b}\\
        & \hbY_{L 2} = \hbX_{L1} \bbW_{L2} + \bbB_{L2}, \label{eq_proof_gin_mlp_per_c}\\
        & \hbX_{L} = \sigma_{L 2} ( \hbY_{L 2} ). \label{eq_proof_gin_mlp_per_d}
    \end{align}
\end{subequations}
Then, the $\ell_2$ norm of difference between the perturbed \eqref{eq_proof_gin_mlp_per_d} and clean outputs \eqref{eq_proof_gin_mlp_ori_d} is
\begin{align}
\label{eq_proof_gin_mlp_layer1}
    \| \hbX_{L} - \bbX_{L} \|
    = \| \sigma_{L 2} ( \hbY_{L 2} ) - \sigma_{L 2} ( \bbY_{L 2} ) \|. 
\end{align}
Using the Lipschitz condition of the nonlinearity function $\sigma_{L 2}(\cdot)$ in \eqref{eq_proof_gin_mlp_layer1}, we have
\begin{align}
    \label{eq_proof_gin_mlp_XL_final}
    \| \hbX_{L} - \bbX_{L} \|
    \leq C_{\sigma_{L 2}} \| \hbY_{L 2} - \bbY_{L 2} \|.
\end{align}
Representing $\hbY_{L 2}$ by \eqref{eq_proof_gin_mlp_per_c} and $\bbY_{L 2}$ by \eqref{eq_proof_gin_mlp_ori_c}, we have
%
\begin{align} \label{eq_proof_gin_mlp_YL2_final}
    \| \hbY_{L 2} - \bbY_{L 2}  \|
    & = \| \hbX_{L1} \bbW_{L2} - \bbX_{L1} \bbW_{L2} \| \notag \\
    & \leq \| \hbX_{L1} - \bbX_{L1} \|\|\bbW_{L2}\|.
\end{align}
Representing $\hbX_{L1}$ by \eqref{eq_proof_gin_mlp_per_b} and $\bbX_{L1}$ by \eqref{eq_proof_gin_mlp_ori_b}, we obtain
\begin{align} \label{eq_proof_gin_mlp_XL_1}
    \| \hbX_{L1} - \bbX_{L1} \|
    = \| \sigma_{L 1} ( \hbY_{L 1} ) - \sigma_{L 1} ( \bbY_{L 1} ) \|.
\end{align}
Using the Lipschitz condition of the nonlinearity function $\sigma_{L 1}(\cdot)$ in \eqref{eq_proof_gin_mlp_XL_1}, we have
\begin{align}
    \label{eq_proof_gin_mlp_XL1_final}
    \| \hbX_{L1} - \bbX_{L1} \|
    \leq C_{\sigma_{L 1}} \| \hbY_{L 1} - \bbY_{L 1} \|.
\end{align}
Representing $\hbY_{L 1}$ by \eqref{eq_proof_gin_mlp_per_a} and $\bbY_{L 1}$ by \eqref{eq_proof_gin_mlp_ori_a}, we have
\begin{align}
    \label{eq_proof_gin_mlp_YL1}
    \| \hbY_{L 1} - \bbY_{L 1} \|
    = \| \hbS \hbX_{L-1} \bbW_{L1} - \bbS \bbX_{L-1} \bbW_{L1} \|.
\end{align}
We can rewrite \eqref{eq_proof_gin_mlp_YL1} by deleting and adding $\bbS \hbX_{L-1} \bbW_{L1}$ as
\begin{align}
    & \hbS \hbX_{L-1} \bbW_{L1} - \bbS \bbX_{L-1} \bbW_{L1} \notag \\
    & = \hbS \hbX_{L-1} \bbW_{L1} - \bbS \hbX_{L-1} \bbW_{L1} + \bbS \hbX_{L-1} \bbW_{L1}  \notag \\
    & \quad - \bbS \bbX_{L-1} \bbW_{L1}\notag \\
    & = (\hbS - \bbS) \hbX_{L-1} \bbW_{L1} + \bbS (\hbX_{L-1} - \bbX_{L-1}) \bbW_{L1}. \label{eq_proof_gin_mlp_addanddelete}
\end{align}
Substituting \eqref{eq_proof_gin_mlp_addanddelete} into \eqref{eq_proof_gin_mlp_YL1}, and using the triangular inequality, we have
\begin{align}
    & \| \bbY_{L 1} - \hbY_{L 1} \| \notag \\
    & \leq \|(\hbS - \bbS) \hbX_{L-1} \bbW_{L1}\| + \|\bbS (\hbX_{L-1} - \bbX_{L-1}) \bbW_{L1} \| \notag \\
    & \leq \| \hbS  - \bbS \| \|\hbX_{L-1}\| \|\bbW_{L1}\| + \|\bbS\| \| \hbX_{L-1}  - \bbX_{L-1} \| \|\bbW_{L1}\|  \label{eq_proof_gin_mlp_YL1_decompose}.
\end{align}
For the second term in \eqref{eq_proof_gin_mlp_YL1_decompose}, we have $\hbX_{L-1} = \bbX_{L-1} = \bbX_0$ for $L=1$.
Then, with the definition of GSO error \eqref{eq:gso_distance}, \eqref{eq_proof_gin_mlp_YL1_decompose} becomes 
\begin{align}
    \label{eq_proof_gin_mlp_YL1_final}
    \| \hbY_{L 1} - \bbY_{L 1} \| \leq \| \bbE \| \|\bbX_{L-1}\| \|\bbW_{L1}\|.
\end{align}
By connecting \eqref{eq_proof_gin_mlp_YL1_final}, \eqref{eq_proof_gin_mlp_XL1_final}, \eqref{eq_proof_gin_mlp_YL2_final}, \eqref{eq_proof_gin_mlp_XL_final} together, we can bound the one-layer GIN output difference as
\begin{align}
    \label{eq_proof_gin_mlp_XL_diff_norm}
    \| \hbX_{L} - \bbX_{L} \|
    \leq C_{\sigma_{L 2}} C_{\sigma_{L 1}} \|\bbW_{L2}\| \|\bbW_{L1}\| \|\bbX_{L-1}\| \|\bbE\|.
\end{align}
Taking the expectation of \eqref{eq_proof_gin_mlp_XL_diff_norm}, we have
\begin{align}
    \mbE\left[ \| \hbX_{L} - \bbX_{L} \| \right] 
    \leq C_{\sigma_{L 2}} C_{\sigma_{L 1}}  \|\bbW_{L2}\| \|\bbW_{L1}\| \|\bbX_{L-1}\|
    \mbE\left[ \| \bbE \| \right].
\end{align}
Finally, let $\xi = C_{\sigma_{L 2}} C_{\sigma_{L 1}}  \|\bbW_{L2}\| \|\bbW_{L1}\| \|\bbX_{L-1}\|$, then, we have 
\begin{align}
    \mbE\left[ \| \hbX_{L} - \bbX_{L} \| \right] 
    \leq \xi \mbE\left[ \| \bbE \| \right].
\end{align}
This completes the proof.
\end{proof}


\bibliographystyle{IEEEtran}
\bibliography{IEEEabrv_added, reference}

\end{document}